%% file: neurips_2026.tex
\documentclass{article}

\usepackage[preprint]{neurips_2026}

\usepackage[utf8]{inputenc} 
\usepackage[T1]{fontenc}    
\usepackage{hyperref}       
\usepackage{url}            
\usepackage{booktabs}       
\usepackage{amsfonts}       
\usepackage{nicefrac}       
\usepackage{microtype}      
\usepackage{xcolor}         

\usepackage[table]{xcolor}
\usepackage{booktabs}
\usepackage{multirow}
\usepackage{arydshln}
\usepackage{amsmath}
\usepackage{natbib}
\usepackage{inconsolata}
\usepackage{amsthm}
\usepackage{cleveref}
\usepackage{graphicx}
\usepackage{minitoc}
\usepackage{adjustbox}
\usepackage{tikz}
\usepackage{subcaption}
\usepackage{enumitem}
\setlist[itemize]{noitemsep}

\definecolor{darkblue}{rgb}{0, 0, 0.5}
\hypersetup{colorlinks=true, citecolor=darkblue, linkcolor=darkblue, urlcolor=darkblue}

\usetikzlibrary{arrows.meta, positioning, fit, backgrounds, calc}

\def\comments{1}
\relax
\newcommand{\ensuretext}[1]{#1}
\if\comments1
    \newcommand{\tempcomment}[4]{\ensuretext{\textcolor{#3}{[\ensuretext{\textcolor{#3}{\ensuremath{^{\textsc{#1}}_{\textsc{#2}}}}} #4]}}}
\else
    \newcommand{\tempcomment}[4]{\ifvmode\else\unskip\fi}
\fi

\newcommand{\oursystem}{\texttt{PreFT}}
\newcommand{\direftp}{DiReFT\textsuperscript{P}}
\newcommand{\direfta}{DiReFT\textsuperscript{A}}
\newcommand{\lorap}{LoRA\textsuperscript{P}}
\newcommand{\direft}{DiReFT}
\newcommand{\loreft}{LoReFT}
\newtheorem{theorem}{Theorem}

\DeclareMathOperator{\sign}{\mathrm{sign}}

\crefformat{section}{\S#2#1#3}
\crefformat{subsection}{\S#2#1#3}
\crefformat{subsubsection}{\S#2#1#3}
\crefformat{paragraph}{\P#2#1#3}
\crefformat{subparagraph}{\P#2#1#3}
\crefmultiformat{section}{\S#2#1#3}{ and~\S#2#1#3}{, \S#2#1#3}{, and~\S#2#1#3}
\crefmultiformat{subsection}{\S#2#1#3}{ and~\S#2#1#3}{, \S#2#1#3}{, and~\S#2#1#3}
\crefmultiformat{subsubsection}{\S#2#1#3}{ and~\S#2#1#3}{, \S#2#1#3}{, and~\S#2#1#3}
\crefmultiformat{paragraph}{\P\P#2#1#3}{ and~#2#1#3}{, #2#1#3}{, and~#2#1#3}
\crefmultiformat{subparagraph}{\P\P#2#1#3}{ and~#2#1#3}{, #2#1#3}{, and~#2#1#3}
\crefrangeformat{section}{\mbox{\S\S#3#1#4--#5#2#6}}
\crefrangeformat{subsection}{\mbox{\S\S#3#1#4--#5#2#6}}
\crefrangeformat{subsubsection}{\mbox{\S\S#3#1#4--#5#2#6}}
\crefrangeformat{paragraph}{\mbox{\P\P#3#1#4--#5#2#6}}
\crefrangeformat{subparagraph}{\mbox{\P\P#3#1#4--#5#2#6}}
\crefname{part}{Part}{Parts}
\Crefname{part}{Part}{Parts}
\crefname{chapter}{Ch.}{Ch.}
\Crefname{chapter}{Ch.}{Ch.}
\crefname{footnote}{Fn.}{Fn.}
\Crefname{footnote}{Fn.}{Fn.}
\crefname{figure}{Figure}{Figures}
\crefname{table}{Table}{Tables}
\crefname{subfigure}{Figure}{Figures}
\Crefname{subfigure}{Figure}{Figures}
\crefname{appsec}{Appendix}{Appendices}
\Crefname{appsec}{Appendix}{Appendices}
\crefname{algocf}{Algorithm}{Algorithms}
\Crefname{algocf}{Algorithm}{Algorithms}

\let\oldappendix\appendix
\renewcommand{\appendix}{\crefalias{section}{appsec}\oldappendix}

\title{\oursystem{}: Prefill-only finetuning for inference efficiency}

\author{%
  Andrew Lanpouthakoun$^*$$^\dagger$ \quad Aryaman Arora$^*$$^\dagger$ \quad Zhengxuan Wu$^\dagger$ \quad Dhruv Pai$^{\ddagger}$ \\
  \textbf{Ben Keigwin}$^{\ddagger}$ \quad
  \textbf{Dan Jurafsky}$^\dagger$ \quad \textbf{Christopher Potts}$^\dagger$\\ 
  $^\dagger$Stanford University \quad $^{\ddagger}$Tilde Research \\
  \texttt{\{andlanpo,aryamana\}@stanford.edu}
}

\begin{document}
\doparttoc 
\faketableofcontents

\maketitle
\vspace{-3em}
\begin{center}
\small
\raisebox{-0.2\height}{\includegraphics[width=1em,height=1em]{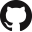}}\hspace{0.5em}\href{https://github.com/stanfordnlp/preft}{\texttt{github.com/stanfordnlp/preft}}
\end{center}

\begin{abstract}

Large language models can now be personalised efficiently at scale using parameter efficient finetuning methods (PEFTs), but \emph{serving} user-specific PEFTs harms throughput, even with specialised kernels and memory management techniques.
This is because, theoretically and empirically, a mismatch exists between prefill (processing a large number of tokens at once) and decode (generating a single token autoregressively): the latter has far lower throughput when serving multiple adapters.
Rather than optimising performance relative to parameter count, for efficient multi-adapter serving, we instead ought to optimise performance relative to \textit{serving throughput}.
We therefore propose \textbf{\oursystem{}} (Prefill-only Finetuning), wherein we only apply the adapter to prefill tokens and discard it afterwards. \oursystem{} significantly increases throughput with minimal effect on performance. We develop and release an efficient implementation of two prefill-only PEFTs, LoRA and ReFT, 
on the vLLM inference engine. We first show that serving multi-user \oursystem{}s is more efficient than traditional PEFTs ($1.9\times$ the throughput when serving $512$ adapters on Llama 3.1 70B). Then, we compare the performance of prefill-only vs.~all-token adapters on a variety of supervised finetuning and reinforcement learning tasks with LMs at varying scales. On SFT, we observe that the evaluation loss of \oursystem{}s is higher than PEFTs, but can be compensated by increasing rank with nearly no reduction in throughput. On RL, we consistently find that \oursystem{}s approach parity with standard PEFTs. Together, this work validates prefill-only adaptation of LLMs as a more favourable accuracy--throughput tradeoff than existing PEFTs for personalised serving.
\end{abstract}

\begin{figure}[!h]
    \centering
    \includegraphics[width=\linewidth]{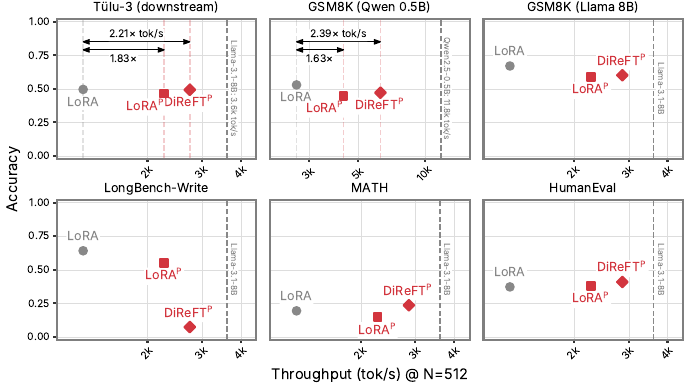}
    \caption{\textbf{\oursystem{} adapters approach parity on a variety of tasks with much better throughput.} Inference throughput vs.~accuracy for PEFTs (LoRA) and \oursystem{}s (\lorap{}, \direftp{}) for a variety of tasks. Tasks besides GSM8K are on Llama 3.1 8B Base. To roughly match parameter count in these plots, LoRA/\lorap{} is always rank-$1$ and \direftp{} is rank-$16$ (Tülu-3, LongBench-Write) or rank-$8$ otherwise. On Tülu-3 (SFT) and for RL tasks (GSM8K, MATH, HumanEval), score is evaluation accuracy on selected downstream benchmarks (see \cref{sec:results}). On LongBench-Write, score is length-following ($S_l$) on 4k--20k token outputs. Inference throughput (tokens/s) is measured on the Uniform Punica microbenchmark when serving 512 adapters on one GPU, simulating multi-user serving (see \cref{sec:vllm}).}
    \label{fig:hero}
\end{figure}

\section{Introduction}

Personalising large language models for billions of users is an open research problem; besides memory and system prompt personalisation, considerable effort has been dedicated to devising parameter-efficient finetuning methods (PEFTs) which reduce the memory footprint of finetuning. LoRA \citep{lora}, the most popular weight-based PEFT, is an attractive option here; it creates lightweight low-rank adapters which can be added to the model's existing weight matrices at inference time. In turn, as systems for serving LLMs have become increasingly optimised, many have sought to make serving thousands of LoRAs compatible with efficient LLM inference \citep[][\textit{inter alia}; see \cref{sec:related}]{punica,sheng2024sloraservingthousandsconcurrent}.

When serving multiple LoRAs in a single inference batch, the biggest challenge is that each batch item demands different sets of LoRA weights; this requires (a) developing custom kernels (e.g.~Punica; \citealp{punica}), since standard batched matrix multiplication is no longer applicable, as it assumes the same matrix is multiplied to each batch item, and (b) managing the memory overhead of keeping multiple LoRAs on GPU, such as by offloading to CPU and paging according to a scheduling algorithm (e.g.~S-LoRA; \citealp{sheng2024sloraservingthousandsconcurrent}). Even with these highly custom optimisations, multi-LoRA serving harms inference throughput.

In this work, rather than optimising what already exists, we take a different approach: we design and implement PEFTs that are inherently suited for efficient LLM inference from the outset. Our key motivation is the distinction between \emph{prefill} (processing the input prompt in parallel) and \emph{decode} (generating outputs sequentially) in efficient LLM inference: prefill is \textit{compute-bound}, so serving PEFTs at this stage does not harm compute utilisation; on the other hand, decode is \textit{memory-bound} and therefore should batch across requests. However, serving many PEFTs in a single batch imposes significant memory overhead and thus erodes the gains from batching. In other words, \textbf{per-request personalisation is more expensive at decode than at prefill}.

To address this challenge, we create efficient \emph{prefill-only} variants of two existing PEFTs, LoRA \citep{lora} and ReFT \citep{wu2024reft}, aptly named \oursystem{}s. Whereas LoRA operates on weights, ReFT intervenes selectively on activations and is thus well-suited for interventions only on selected tokens.
We implement efficient multi-adapter versions of these methods in a fork of the vLLM inference engine \citep{vllm}, and further add support for our inference backend in the \texttt{trl} reinforcement learning library. Our experimental results provide strong support for two key claims, which we summarise in \cref{fig:hero}.

First, we show that \textbf{\oursystem{}s are efficient at inference} (\cref{sec:vllm}). We benchmark our multi-\oursystem{} serving system against existing multi-LoRA serving at all token positions in vLLM: without any custom kernels, prefill-only ReFT achieves $1.9\times$ the throughput of multi-LoRA when serving up to $512$ adapters on \mbox{Llama 3.1 70B} with tensor parallelism across $4$ H100s, and prefill-only LoRA achieves $1.87\times$. Gains hold on a variety of model scales ($2.21\times$ / $1.83\times$ on Llama 3.1 8B, $2.39\times$ / $1.63\times$ on Qwen 2.5 0.5B for \direftp{} / \lorap{}, respectively).

Second, we show that \textbf{\oursystem{}s maintain strong performance} (\cref{sec:results}). On SFT, \oursystem{}s have higher evaluation loss than all-position PEFTs on Tülu-3 and OpenThoughts in data- and parameter-matched training, but find no statistically significant difference in accuracy on downstream evaluations. On RL tasks, \oursystem{}s approach PEFTs on math reasoning with GSM8K and MATH \citep{gsm8k, math} and code generation trained with MBPP and evaluated on the HumanEval dataset \citep{mbpp, humaneval}, but overall a small difference remains. On SFT for long-form text generation with trained with LongWriter and evaluated with LongBench-Write \citep{longwriter}, we find that \lorap{} maintains similar length-control ability and quality as traditional LoRA.

Overall, we find that \oursystem{}s are worth deploying over traditional all-position PEFTs in the multi-user personalisation setting: they maintain most or all of the performance of traditional PEFTs while significantly improving inference efficiency.

\section{Related work}
\label{sec:related}

\paragraph{Parameter-efficient finetuning (PEFT).} A variety of PEFT techniques have been proposed in the literature; broadly, they freeze the original model and only train some lightweight parameters applied as additional operations. We focus on LoRA \citep{lora} and ReFT \citep{wu2024reft}; the former applies low-rank adapters to selected weight matrices, and the latter applies low-rank hidden representation edits to selected token positions and layers. Other approaches add entire model components \cite[][\textit{inter alia}]{houlsbyParameterEfficientTransferLearning2019,pfeiffer-etal-2020-mad} or learnable soft tokens \citep{li-liang-2021-prefix}. Additionally, many variants of LoRA \citep{dora,vera,loraxs,morris2026learningreason13parameters} and ReFT \citep{zeng-etal-2025-towards-context,wang2025commandvpastingllmbehaviors} have been proposed which modify parametrisation and/or initialisation. A few works apply mask LoRA on selected token positions \citep{alora,dietz2026splitpersonalitytrainingrevealing},\footnote{In their appendix, the latter even suggests a masked neuron-level adapter, which is similar in spirit to ReFT.} but none propose prefill-only application. We choose to focus on LoRA for its widespread adoption in serving infrastructure, and we choose ReFT because it is designed for token-level interventions. 

\paragraph{Serving personalised LLMs.} Serving multiple finetunes of one model adds overhead to LLM inference. While LoRA adapters can be merged into and unmerged from model weights as in dLoRA \citep{wu2024dlora}, if done naïvely this reduces throughput and increases latency in \textit{multi}-LoRA serving. Therefore, S-LoRA \citep{sheng2024sloraservingthousandsconcurrent} and Punica \citep{punica} proposed serving \textit{unmerged} LoRA adapters with specialised kernels for multi-LoRA batch matrix multiplication, along with offloading LoRAs to CPU memory and using scheduling algorithms (e.g.~paging in S-LoRA and dLoRA) to fetch them into GPU as needed while minimising latency and idle compute. Later work introduces further memory management and compression improvements for multi-LoRA serving \citep{gabrielsson2025compress,zhang2025improving,zhu2025cannikin,dbrxfastpeft,ni2026plora}. \oursystem{}s are able to make use of all of these optimisations.



\section{Introducing \oursystem{}}
\label{sec:intro}
We now introduce \oursystem{} as a simple modification of existing all-position adapters. For background, we first rigorously define the two PEFTs that we experiment with: LoRA \citep{lora} and ReFT  \citep{wu2024reft} under a unified structure. We then define our approach and elaborate on the efficiency motivation for prefill-only adapters.

\subsection{Adapter methods}
We first define an adapter $\mathcal{A}_\phi$ on a frozen base model $\mathcal{M}$ as a tuple $\mathcal{A}_\phi = \langle \Phi_\phi, \mathcal{F}, \mathcal{P} \rangle$ consisting of:
\begin{itemize}
    \item an \emph{intervention function} $\Phi_\phi$ with trainable parameters $\phi$,
    \item a set of \emph{target modules} $\mathcal{F}$ within $\mathcal{M}$ at which the intervention is applied,
    \item a set of \emph{target positions} $\mathcal{P}$ within each input sequence at which the intervention is active.
\end{itemize}
Each target module $f \in \mathcal{F}$ is a computation within $\mathcal{M}$ that takes an input $\mathbf{x}$ and produces an output $f(\mathbf{x})$. One could imagine a weight matrix applied to its input, or the identity map on the residual stream at a given layer. The adapter replaces $f(\mathbf{x})$ with $\Phi_\phi(f(\mathbf{x}))$ at positions $i \in \mathcal{P}$:
\begin{equation}
f(\mathbf{x})_i \leftarrow \begin{cases} \Phi_\phi(f(\mathbf{x})_i) & \text{if } i \in \mathcal{P} \\ f(\mathbf{x})_i & \text{otherwise} \end{cases}
\end{equation}
The base parameters $\theta$ remain frozen; only $\phi$ is trained. Different adapter methods correspond to different choices of what kind of module $\mathcal{F}$ ranges over and what form $\Phi_\phi$ takes.

\paragraph{LoRA.} LoRA \citep{lora} targets frozen weight matrices: each $f \in \mathcal{F}$ is a linear map $f(\mathbf{x}) = W \mathbf{x}$ with $W \in \mathbb{R}^{n \times m}$. LoRA introduces trainable matrices $\mathbf{B} \in \mathbb{R}^{n \times r}$ and $\mathbf{A} \in \mathbb{R}^{r \times m}$ with $r \ll \min(n, m)$, and defines:
\begin{equation}
\Phi^{\text{LoRA}}_\phi(f(\mathbf{x})) = f(\mathbf{x}) + \alpha \mathbf{B}\mathbf{A} \mathbf{x}
\end{equation}
where $\alpha$ is a constant scaling prefactor and $\phi = \{(\mathbf{A},\mathbf{B})\}_{f \in \mathcal{F}}$. We work in the experimental setting where the choices for $\mathcal{F}$ are the attention projections and MLP projections at all layers. In standard LoRA, $\mathcal{P}$ spans all token positions.

\paragraph{ReFT.} ReFT \citep{wu2024reft} targets residual streams: each $f \in \mathcal{F}$ is the identity map at a given layer $l$, so $f(\mathbf{x}) = \mathbf{h}^l$. We focus on the DiReFT parametrisation, which we found most effective:
\begin{equation}
\Phi^{\text{DiReFT}}_\phi(\mathbf{h}) = \mathbf{h} + \mathbf{B}^\top(\mathbf{A} \mathbf{h} + \mathbf{b})
\end{equation}
where $\mathbf{A}, \mathbf{B} \in \mathbb{R}^{r \times d}$ and $\mathbf{b} \in \mathbb{R}^r$, with $r < d$, and $\phi = \{(\mathbf{A}, \mathbf{B}, \mathbf{b})\}_{f \in \mathcal{F}}$. Since $\mathbf{A}$ is unconstrained, $\Phi^{\text{DiReFT}}$ adds a vector lying in an at-most rank-$r$ subspace to the residual stream. In the original ReFT formulation, $\mathcal{P}$ was a tunable hyperparameter consisting of $k$ prefix and suffix prompt positions; in this work, we instead let $\mathcal{P}$ span either all prompt tokens or all token positions, as defined next.

\subsection{\oursystem{}}
We now introduce \oursystem{} as a single-line modification to existing adapters: we restrict the position set to the prompt tokens only,
\begin{equation}
    \mathcal{P}^{\text{PreFT}} = \{1, \dots, p\},
\end{equation}
where $p$ is the length of the input prompt. Applied to LoRA and DiReFT, this yields the variants \lorap{} and \direftp{}, which we compare against their all-position counterparts LoRA and \direfta{} (with $\mathcal{P}^A = \{1, \dots, p+T\}$ where $T$ is the number of generated tokens).

\paragraph{Why prefill-only adapters?} Prefill and decode steps represent different inference workloads: prefill is \textit{compute-bound} (since many tokens from a single request can be processed in parallel, occupying the available compute), and decode is \textit{memory-bound} (since each request only computes a single new token and loads a potentially large KV cache from memory).\footnote{The exact arithmetic intensity of each operation depends on implementation and hardware, but refer to e.g.~\citet{scaling-book}: ``During prefill, all matrix multiplications are basically always compute-bound [\ldots] During generation, the total token batch size must be greater than $B_{\text{crit}}$ to be compute-bound on the linear/feed-forward operations ($240$ for bf16 params on TPU v5e). Because generation happens serially, token-by-token, this requires us to batch multiple requests together, which is hard!''} This distinction necessitates different kinds of optimisation in order to maximise efficiency on both types of workloads. As a result, disaggregating prefill and decode workers entirely is increasingly common \citep[][\textit{inter alia}]{zhong2024distserve,mooncake,qin2026prefillasaservicekvcachenextgenerationmodels}.

To maximise throughput during decode, it is beneficial to batch multiple requests. However, if each request requires serving a unique PEFT adapter, this adds to memory overhead since (like per-request KV caches) each PEFT for the batch must be loaded into memory. To be clear, \citet{lorafusion} show that LoRA's down projection (where $r$ is rank; $d$ is output dimension, and $b$ is batch size) is already \textit{memory-bound} in half-precision:
\begin{equation}
    \mathbb{I}_{\text{LoRA,down}} = \frac{\text{FLOPs}}{\text{bytes}} = \frac{2bdr}{2(bd + dr + br)} = \frac{1}{\frac{1}{r} + \frac{1}{b} + \frac{1}{d}} \ll \mathbb{B}
\end{equation}
where $\mathbb{B}$ is e.g.~295 for FP16 on NVIDIA H100. The down projection for DiReFT is identical and therefore also memory-bound. This adds significant overhead at already memory-bound decode.

We therefore simply propose \textbf{not applying the memory-bound PEFT operations at decode steps.} This is in addition to existing serving optimisations such as custom multi-LoRA kernels and memory paging \citep{punica, sheng2024sloraservingthousandsconcurrent}: \oursystem{} removes the decode-time cost entirely rather than amortising it, and composes naturally with these techniques when they are applied to prefill. It remains for us to demonstrate that the resulting efficiency gains are worth any cost in downstream quality.


\section{Efficient multi-\oursystem{} inference}
\label{sec:vllm}

We now implement and benchmark multi-\oursystem{} serving with LoRA and ReFT in a fork of the LLM inference engine vLLM \citep{vllm}. We compare our approach against vLLM's built-in multi-LoRA serving, and show that prefill-only adapters are highly performant under a variety of inference workloads. For example, on \texttt{Llama-3.1-70B} we achieve $1.9\times$ and $1.87\times$ more inference throughput than multi-LoRA for prefill-only ReFT and LoRA, respectively. Notably, unlike for LoRA, we do not use specially designed kernels to optimise multi-ReFT serving, but we still achieve significant efficiency gains.

\begin{figure*}[t]
  \centering
  \begin{subfigure}[b]{0.48\linewidth}
    \includegraphics[width=\linewidth]{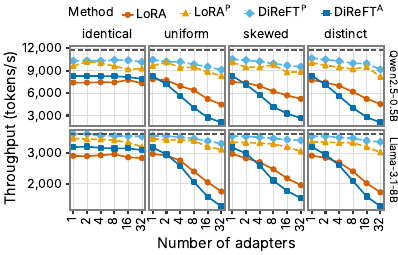}
    \caption{Throughput for LoRA, \direftp{}, and \direfta{} under 4 inference workloads with varying num.~adapters.}
    \label{fig:punica-on-gpu}
  \end{subfigure}\hfill
\begin{subfigure}[b]{0.48\linewidth}
  \includegraphics[width=\linewidth]{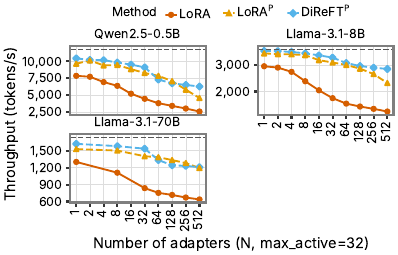}
      \caption{Throughput for LoRA vs.~\oursystem{}s in the \textbf{Uniform} setting, scaling up to 512 adapters with paging.}
  \label{tab:punica-512}
\end{subfigure}
  \caption{\textbf{\oursystem{}s maintain throughput more effectively than traditional PEFTs as number of adapters increase.} Inference throughput (tokens/s) on the Punica microbenchmark when comparing rank-$1$ LoRA (prefill-only vs.~all positions) and rank-$8$ DiReFT (prefill-only vs.~all positions) with varying numbers of adapters.}
  \label{fig:punica-single-gpu}
\end{figure*}

\paragraph{Implementation.}
To make \oursystem{} usable at inference time, we fork \texttt{vLLM} and integrate ReFT and \lorap{} adapters directly into the serving engine. The fork also adds compatibility for both multi-adapter and single adapter runs and exposes a weight-sync interface that pushes updated adapter parameters to live inference workers in a single collective call, enabling tight training--inference loops for on-policy RL without restarting the engine. The implementation works transparently under chunked prefill and across the attention kernels vLLM ships. \footnote{Full details, including the position-mask and CUDA-graph considerations that make this work, are provided in Appendix~\ref{app:vllm}.}

\subsection{Benchmarking}

We benchmark the throughput of \oursystem{}s against the multi-LoRA implementation in vLLM, using the microbenchmarks from Punica \citep{punica}. vLLM implements both multi-LoRA memory paging using the approach proposed in S-LoRA \citep{sheng2024sloraservingthousandsconcurrent} and custom Triton kernels for heterogeneous batch LoRA processing based on Punica.\footnote{See \url{https://github.com/vllm-project/vllm/pull/5036}.}

\paragraph{Single GPU, all adapters on-GPU.} The single GPU microbenchmark in Punica evaluates throughput given $1,000$ requests processed in a first-come/first-served manner (not offline). We replicate this microbenchmark on a single H100 80GB GPU, for varying numbers of adapters. All adapters are kept in GPU memory in this setting so as to limit the comparison to inference efficiency. The maximum batch size for the vLLM worker is set to $32$. We test four types of workloads: \textbf{Identical} (all requests to one adapter; overhead is from holding adapters in memory), \textbf{Uniform} (adapter is chosen uniformly at random), \textbf{Distinct} (all requests are to different adapters), and \textbf{Skewed} (adapter is sampled from a Zipfian distribution). We run this benchmark on \texttt{Qwen2.5-0.5B Instruct} and \texttt{Llama-3.1-8B Instruct}, comparing rank-$1$ LoRA, rank-1 \lorap{}, rank-$8$ \direftp{}, and rank-$8$ \direfta{}. We vary the number of adapters in $\{1, 2, 4, 8, 16, 32\}$.

In \cref{fig:punica-on-gpu}, we plot inference throughput with varying numbers of adapters in all settings, with a grey dashed line indicating baseline throughput without any adapters ($11,792$ tok/s for \texttt{Qwen2.5-0.5B}, $3,606$ tok/s for \texttt{Llama-3.1-8B}). We find that \direftp{} maintains near-baseline throughput even as the number of adapters is increased to $32$; it achieves $2.08\times$ more throughput than LoRA in the Uniform setting with \texttt{Qwen2.5-0.5B} and $1.93\times$ with \texttt{Llama-3.1-8B}.

\paragraph{Single GPU, paged memory for moving adapters from CPU to GPU.} We replicate vLLM's  multi-LoRA memory paging system (adapted from S-LoRA) for multi-ReFT, which allows keeping adapters on CPU and moving only a subset onto GPU as needed. We fix the maximum number of adapters on-GPU to $32$, and scale up the Punica single-GPU benchmark in the \textbf{Uniform} setting to $512$ adapters. \Cref{tab:punica-512} shows that \direftp{} continues to grow its throughput lead over LoRA even with paged adapter memory.

\paragraph{Multiple GPUs, paged memory.} Finally, we use tensor parallelism over $4$ GPUs to benchmark a larger model: \texttt{Llama-3.1-70B} (dense). \Cref{tab:punica-512} includes this and shows similar throughput trends.

\section{The effectiveness of \oursystem{}s}
\label{sec:results}

\begin{figure}
    \centering
    \includegraphics[width=\linewidth]{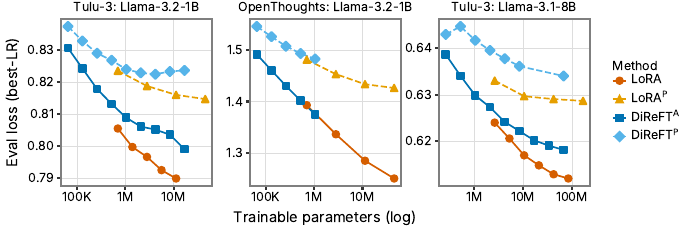}
    \caption{\textbf{\oursystem{}s underperform all-position PEFTs at matched parameter counts, but scale predictably with rank.} Eval loss (best-LR) for \lorap{} and \direftp{} compared to their all position counterparts, on Tülu-3 (\texttt{Llama-3.2-1B} and \texttt{Llama-3.1-8B}) and OpenThoughts \texttt{Llama-3.2-1B}, as a function of trainable parameter count.}
    \label{fig:sft_main}
\end{figure}

Having established the inference efficiency benefits of \oursystem{}s in \cref{sec:vllm}, we now turn to validating the performance of prefill-only ReFT and LoRA relative to their all-position counterparts. We perform the following experiments: 
\begin{itemize}
    \item SFT on instruction-following with Tülu-3 \citep{lambert2025tulu3pushingfrontiers}, with downstream evaluations on held-out benchmarks (\cref{sec:sft});
    \item RLVR on math (GSM8K and MATH) and code reasoning (MBPP, evaluations on HumanEval), with comparisons to SFT performance (\cref{sec:math}; also see \cref{sec:gsm8k-more});
    \item SFT on long-form text generation with LongWriter and evaluated on LongBench-Write (\cref{sec:long_outputs});
\end{itemize}
From the first two sets of experiments, we find that \oursystem{}s trained with SFT have higher evaluation loss but achieve parity on downstream task performance compared to all-position ReFT and LoRA; additionally, on RLVR tasks, \oursystem{}s achieve parity on MATH and HumanEval, but not GSM8K.
We further consider that \oursystem{} performance may degrade over longer generations; in experiments on LongWriter, we find that \lorap{} maintains the adapted behaviour of its all positions counterpart and investigate why \direftp does not do the same.

Overall, we find that \oursystem{}s approach downstream performance of PEFTs on RL and SFT and the gap on long outputs exists only on \direftp{}, not \lorap{}. Therefore, \oursystem{}s sacrifice little performance compared to PEFTs despite large efficiency gains.

\subsection{SFT: Tülu-3 and OpenThoughts}
\label{sec:sft}
\input{tabs/sft_prefill_vs_allpos}

We seek to compare performance between \oursystem{} and all-position PEFTs under supervised finetuning. To do so, we replicate experiments from \citet{schulman2025lora}, as a useful testbed for comparing adapter performance. We compare \direftp{} and \direfta{}, and \lorap{} and LoRA.

\paragraph{PEFTs have better evaluation loss than \oursystem{}s.} Initially, in keeping with \citeauthor{schulman2025lora}, we finetune \texttt{Llama-3.2-1B Instruct} and \texttt{Llama-3.1-8B Instruct} on Tülu-3 (\citealp{lambert2025tulu3pushingfrontiers}; a general-purpose instruction-tuning dataset) and OpenThoughts (\citealp{guha2025openthoughtsdatarecipesreasoning}; a reasoning dataset with much longer generation lengths). We truncate Tülu-3 to $2{,}048$ tokens and OpenThoughts to $16{,}384$ tokens, and only train on completions; further details are in \cref{app:sft_infra}.

In \cref{fig:sft_main}, we show the results of our sweeps, plotting the evaluation loss of the best-performing LR for each method at each rank across the three experimental settings. In general, we find that \oursystem{}s have \textit{worse evaluation loss} than all-position PEFTs, but show improvement as rank is scaled up with little effect on throughput.

\paragraph{PEFTs and \oursystem{} have similar downstream performance despite loss differences.} On Llama Instruct models, finetuning hardly changes downstream performance due to a high starting performance, so we instead run SFT with Tülu-3 only on \textit{base} \texttt{Llama-3.2-1B} and \texttt{Llama-3.1-8B} with a similar LR and rank sweep. We evaluate on IFEval, MMLU, and GSM8K, which test various downstream abilities; see \cref{app:sft_infra} for more information.
In \cref{tab:sft_prefill_vs_allpos}, we compare the best-overall LR (by eval.~loss) parameter-matched \oursystem{}s and PEFTs using the Cochran--Mantel--Haenszel test\footnote{Details behind motivation and implementation of this test are in \cref{app:statistics}.} stratified over the evaluation benchmarks and report per-comparison statistical significance. Then, we run the two-sided Wilcoxon signed-rank test on the per-comparison pairs, ultimately finding that \textbf{there is no statistically significant difference between \oursystem{}s and PEFTs on downstream tasks under SFT} ($p = 0.083$).

\subsection{RLVR: Math reasoning and code generation}
\label{sec:math}

We now continue our comparison on reinforcement learning with verifiable feedback (RLVR) tasks, focusing initially on math reasoning and code generation. We use a fork of the \texttt{trl} library for all reinforcement learning experiments, with support for our \texttt{vllm} fork as the inference backend. We use the GRPO \citep{shao2024deepseekmathpushinglimitsmathematical} algorithm to train \texttt{Qwen2.5-0.5B} and \texttt{Llama-3.1-8B}  on the train set of GSM8K, MATH, and MBPP respectively; further details are in \cref{app:rl_infra}. 


\paragraph{Results.} \cref{tab:rl_new_all_tasks_by_rank} reports prefill-only and all-position evaluation accuracy across GSM8K, MATH, and HumanEval on \texttt{Llama-3.1-8B} and \texttt{Qwen2.5-0.5B}. On \texttt{Qwen2.5-0.5B} MATH evaluations, all six prefill cells fall within $2.08$pp of their all-position counterpart and four favor prefill (two of which are notably statistically significant). HumanEval results look similar to MATH results and show parity amongst \oursystem{}s and PEFTs. However, $\Delta$s happen to be much larger across GSM8K evaluations than other groups; we further investigate this in \cref{sec:gsm8k-more} but ultimately do not have a clear explanation. Under the two-sided Wilcoxon signed-rank test, we do find a statistically significant difference between \oursystem{}s and PEFTs, with $p = 0.014 < 0.05$. With a mean $\Delta = -2.04$, the difference is not large but still consistently in favour of all-position PEFTs.

\input{tabs/rl_new}

\subsection{Characterising limits of prefill-only attenuation}
\label{sec:long_outputs}
We hypothesise that the influence of prefill-only adapters attenuates with distance and should produce coherent output at any length but progressively lose the ability to enforce behaviour as generation continues. We therefore evaluate the extent to which the prefill-only intervention's effect has faded and the base model's original behaviour reasserted itself. LongWriter \citep{longwriter} is well-suited to test this hypothesis. The benchmark prompts models to produce text at specified target lengths ranging from a few hundred to tens of thousands of tokens, and scores each generation along two independent axes: a quality score $S_q$ assessing the produced text, and a length-following score $S_l$ measuring whether the output matched the requested length.\footnote{Precise descriptions of the scoring methods and the training setup are in \cref{app:longwriter-scoring}.} We train \oursystem{}s and PEFTs with SFT at rank $16$ on \texttt{Llama-3.1-8B Instruct} and evaluate at four output-length brackets. 

\paragraph{LoRA-prefill preserves length-following; DiReFT-prefill overshoots.}
In \cref{fig:longwriter_results}, \lorap{} does not fade with decode
distance; rather its curves run parallel to LoRA throughout, rejecting our
hypothesis. \direftp{} confirms the hypothesis but fails in the
\emph{opposite} direction: rather than reverting to base behaviour at
long lengths, its conditional length-matching signal collapses into
an unconditional \emph{write long} bias.
We conduct ablations on \lorap{} specifically in \cref{app:output_attenuation}, finding that applying it to (entire or only parts of) MLPs still maintains high length-following capability, showing that in theory per-position operations are sufficient. We leave further investigation of length collapse in \direftp{} to future work. Overall, we find that \lorap{} is sufficient for maintaining high performance on this task.


\begin{figure}
    \centering
        \includegraphics[width=\linewidth]{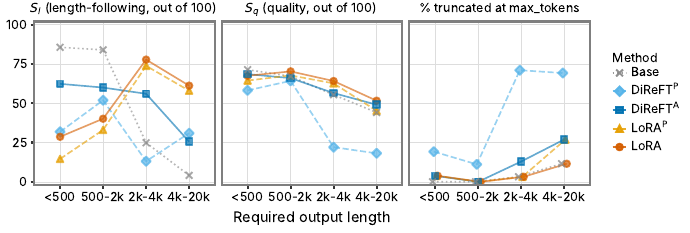}
    \caption{\textbf{\lorap{} matches LoRA on length-following
without sacrificing quality; \direftp{} does not.}
LongWriter on Llama-3.1-8B-Instruct (rank $16$), four required
output-length brackets. The prefill and all-position LoRA variants are
indistinguishable on both $S_l$ (left) and $S_q$ (middle) at every
bracket. \direftp{} instead overshoots target length, writing
runaway generations that saturate the 32k-token decoding cap on most
$\geq$2k requests (right) and dragging $S_q$ below Base at long brackets.
\direfta{} lifts $S_l$ over base across all brackets while
preserving quality.}
    \label{fig:longwriter_results}
\end{figure}

\section{Discussion}
\label{sec:discussion}
\paragraph{Performance vs.~throughput as a new goal for PEFTs.} Parameter-efficient finetuning research optimises performance per trainable parameter, but our results suggest this is the wrong axis for personalised serving. Parameter count is a proxy for memory footprint at training time; what matters at \textit{inference} time is whether adapter operations land in the compute-bound prefill phase or the memory-bound decode phase. \oursystem{}s are the first to make this distinction, and the throughput gains we measure are the result of adapter operations being placed where they cost the least.

\paragraph{Relationship with KV cache-based adaptation techniques.} Besides standard adapter-style PEFTs (e.g.~LoRA) and representation-editing PEFTs (e.g.~ReFT), another approach to cheap adaptation modifies the KV cache, via learnable soft prompts \citep{li-liang-2021-prefix} or by directly learning new KV cache entries \citep{eyuboglu2025cartridgeslightweightgeneralpurposelong}. These techniques also avoid loading in adapters, at the expense of having to store and compute attention for a larger KV cache (which does add overhead on decode steps). \oursystem{}s do not add new entries to the KV cache, but both our approaches disable prefix caching across different adapters. We do not benchmark against these methods directly; based on their KV-cache overhead at decode, we expect their throughput to fall between PEFTs and PreFTs, but confirming this empirically is left to future work. These approaches are thus complementary to ours and may occupy a novel position along the accuracy--throughput tradeoff.

\paragraph{Future work.}
We emphasise that LoRA and DiReFT were originally designed assuming different position applications. In this work, we evaluate the prefill-only restriction using these existing adapter architectures unchanged, but novel architectures for such adapters are unexplored. For example, an architecture that targets attention projections to write its effect into the KV cache, or one whose parametrisation is optimised for fixed-context-length operations; these could plausibly close the remaining performance gaps we observe in \cref{sec:results}. We leave the design of such adapters to future work.

\section{Conclusion}

We introduce \oursystem{}, a family of prefill-only parameter-efficient finetuning methods designed for efficient multi-adapter serving. Our throughput gains of e.g.~$1.9\times$ the throughput of multi-LoRA when serving 512 adapters on \texttt{Llama-3.1-70B} come from placing adapter operations in more optimal positions, rather than kernel optimisations. Across model scales from 0.5B to 70B parameters, prefill-only adaptation matches all-position adapters on SFT downstream evaluations and approaches them on RLVR tasks, along with significant throughput gains. For the workloads most likely to drive personalised serving at scale (e.g.~style adaptation, per-user memory), \oursystem{} sits on a more favourable point of the accuracy--throughput frontier than existing PEFTs.

\section*{Acknowledgements}

We thank Harshit Joshi, Dilara Soylu, Yanzhe Zhang, Nikil Roashan Selvam, Jiuding Sun, and members of the Stanford \texttt{\#weekly-interp-meeting} for helpful discussion throughout the project. We thank the Palo Alto location of Molly Tea for enabling our learning rate sweeps.

\bibliographystyle{plainnat}
\bibliography{main,neurips}


\newpage
\appendix
\renewcommand \thepart{}
\renewcommand \partname{}
\noptcrule
\part{Appendix} 
\parttoc 
\newpage
\section{Implementation details}
\label{app:impl}

\subsection{LoRA}
We use the \texttt{peft} library for our LoRA finetuning. We follow the
parametrisation of LoRA used in \citet{schulman2025lora}: we set the scaling
prefactor to $\alpha / r$ with $\alpha := 32$, use Kaiming uniform
initialisation for $\mathbf{A}$, and initialise $\mathbf{B}$ as zeroes.
Unless otherwise noted, we apply LoRA to the attention projections
$\{W_q, W_k, W_v, W_o\}$ and MLP projections
$\{W_{\text{gate}}, W_{\text{up}}, W_{\text{down}}\}$ at all layers. We use
a LoRA dropout of $0$ throughout. This is because in RL, a non-zero dropout desyncs the training and rollout policies because vLLM does not see the dropout mask.
For \lorap, we use the same parametrisation and initialisation
as standard LoRA; the prefill-only restriction is enforced by
position-routing logic. Specific implementation logic differs depending on training setup. For
reinforcement learning, this is described in \cref{app:vllm}, and for
SFT, this is described in \cref{app:sft_infra}.

\subsection{ReFT}
We use a refactored version of the \texttt{pyreft} library, documented in
\cref{app:reftmodel}. Following the parametrisation we develop in
\cref{app:zero_delta}, we initialise $\mathbf{A}$ as a random matrix with
orthonormal rows, and $\mathbf{B}$ and $\mathbf{b}$ as zeroes; this gives
zero-delta initialisation (the untrained adapter applies the identity).
We apply the scaling prefactor $s_r = 1/\sqrt{r}$ from \cref{app:lr_transfer} to
enable learning rate transfer across ranks. We always
intervene on the residual stream at every layer, post-MLP. For prefill-only
\direftp, we use the same parametrisation and initialisation; only the
position set $\mathcal{P}$ described in \cref{sec:intro} differs. As with LoRA, the prefill-only restriction at training time is documented in \cref{app:vllm} for
reinforcement learning experiments.

\subsection{Supervised finetuning}
\label{app:sft_infra}

\subsubsection{Infrastructure}
For SFT, we use a patched HuggingFace \texttt{Trainer} class. Unless overridden
by a specific experiment, all SFT runs share the following settings:
\begin{itemize}
  \item \textbf{Optimiser.} AdamW with $(\beta_1, \beta_2) = (0.9, 0.999)$ and
        weight decay $0$
  \item \textbf{Schedule.} Constant learning rate. We use no warmup for
        the reasoning sweeps (GSM8K, MATH) and a warmup ratio of $0.1$
        for the long-context (LongBench-Write) sweeps.
  \item \textbf{Batch size.} $32$ sequences per optimiser step;
        this is achieved as per-device batch size $2$ with $16$ gradient
        accumulation steps on a single GPU, or per-device batch size $1$
        with $16$ gradient accumulation steps on two GPUs.
  \item \textbf{Mixed precision.} bf16 throughout; we never run ReFT in fp16.
  \item \textbf{Sequence length.} $2{,}048$ tokens for reasoning,
        $32{,}768$ tokens for LongBench-Write.
  \item \textbf{Kernels.} Flash-Attention~2 and the Liger fused
        linear--cross-entropy kernel are enabled for every run; gradient
        checkpointing is enabled for all 8B runs and disabled at smaller
        scales.
\end{itemize}

\subsubsection{Datasets}
\label{sec:datasets}

For Tülu-3 SFT, we evaluate on the following relatively simple (and thus suitable for our model scales) downstream tasks which test the ability to follow instructions along with knowledge and reasoning:

\begin{itemize}
    \item \textbf{GSM8K} \citep{gsm8k} consists of grade-school math word problems and usually requires multi-step reasoning to solve.
    \item \textbf{MMLU} \citep{mmlu} contains several domains of multiple-choice world knowledge and problem-solving questions.
    \item \textbf{IFEval} \citep{ifeval} contains a variety of verifiable constraint-satisfying tasks.
\end{itemize}

\subsection{Reinforcement learning with verifiable rewards}
\label{app:rl_infra}
\paragraph{Learning algorithm.} We use the standard GRPO
\citep{shao2024deepseekmathpushinglimitsmathematical} algorithm for RLVR,
as implemented in \texttt{trl}. We sample $K = 4$ completions per group
with temperature $1.0$ and top-$p = 1.0$, use a per-device batch size of
$4$ prompts with $8$ gradient accumulation steps for an effective $32$
prompts per optimiser step, train for one epoch with a warmup ratio of
$0.1$, and do not use a KL penalty (i.e.~$\beta = 0$). Maximum prompt
length is $512$ tokens for GSM8K and $2{,}048$ tokens for MATH; maximum
completion length is $2{,}048$ tokens for GSM8K and $4{,}096$ for MATH.

\subsubsection{Reward function} 
\label{app:reward_function}
Across all three tasks we use a binary
accuracy reward, of $1.0$ for correct answers and $0.0$ otherwise, with no
format reward, intermediate reasoning credit, or length shaping. Train-time
and eval-time scoring share the same code path so that reward magnitudes
correspond directly to test accuracy. We will continue by describing the unique reward functions for each dataset trained.

\paragraph{MATH.} We extract the predicted answer with the
\texttt{math\_verify} library, which anchors on the last
\verb|\boxed{...}| in the completion and compares to the gold answer for
symbolic equivalence (which handles LaTeX formatting). On
\texttt{math\_verify} parse failures we fall back to a legacy
\verb|\boxed{...}|-aware extractor combined with a string match that
strips \verb|$|, commas, and whitespace.

\paragraph{GSM8K. } GSM8K answers typically are integers or simple decimals, so we use the canonical extractor: split on the \verb|####| delimiter and take the first numeric match after it, falling back to the last number in the completion when
the delimiter is absent.

\paragraph{HumanEval.} Here, the reward is unit-test based: we splice the model's completion into the function signature provided in the prompt, execute the resulting program against the task's hidden test cases in a sandboxed subprocess, and award $1.0$ reward if every test asserts. Syntax errors, runtime exceptions, and timeouts all return $0.0$ reward. There is no partial credit for ``code that compiles''.

\subsection{Inference}
\label{app:inference_implementation}
For all benchmarking and serving experiments, we use a fork of
\texttt{vllm} with multi-LoRA and multi-ReFT support. Implementation
details of the position-mask buffer, CUDA-graph
integration, and the weight-sync RPC for on-policy RL for the fork are in
\cref{app:integration,app:graph-safe,app:weight-sync}.

\subsection{Evaluation infrastructure}
\label{app:eval_infra}
After training, we evaluate with greedy sampling (temperature $0.0$),
generating a single completion per evaluation prompt and computing
rewards identically to training. We evaluate every checkpoint on the
full test split rather than a held-out subsample, except where noted in
the experiment-specific appendices.

\subsection{Benchmarking infrastructure}
\label{app:bench_infra}
For all serving benchmarks, we replicate the synthetic workload
introduced by Punica \citep{punica}: prompt and output lengths
are drawn from the lognormal/uniform distribution they propose, and
each request is tagged with one of $N$ adapters under a chosen
popularity mix (uniform unless noted, with skewed and identical mixes
used in ablations). We drive our \texttt{vllm} fork directly so we can
attribute wall time to the prefill and decode phases separately,
admitting new requests into the engine one at a time up to a fixed
concurrency $B$. Each run is preceded by a warmup pass over the same
workload to absorb CUDA-graph capture and any one-off setup costs, so
the timed numbers reflect steady-state serving rather than start-up
overhead. We report per-token prefill and decode latency (median, with
$p_{90}$ and $p_{99}$ tails) and end-to-end token throughput, all
measured in bf16. The full workload, engine configuration, and per-figure sweep settings are documented in \cref{app:benchmarking}.

\subsection{Hardware}
\label{app:hardware_usage}
SFT and RLVR experiments on Llama-3.2-1B-Instruct and Qwen2.5-0.5B were
run on a single H100~80GB GPU per job (training and the colocated vLLM
engine share the GPU; vLLM is allocated $40\%$ of GPU memory and the
trainer takes the remainder). Llama-3.1-8B RLVR runs use four
H100~80GBs per job: vLLM serves the policy on two GPUs with tensor
parallelism (\texttt{TP}=2) over an HTTP/NCCL bridge, and the trainer
runs on the other two GPUs under DDP (per-device batch size $4$,
gradient accumulation $8$, effective $32$ prompts per step over the two
data-parallel ranks). Llama-3.1-8B SFT runs use two H100s with FSDP
\texttt{FULL\_SHARD} and transformer-based auto-wrap. Multi-GPU
benchmarks for Llama-3.1-70B and Qwen3-30B-A3B use four H100s with
tensor parallelism, as described in \cref{app:benchmarking}.

\newpage
\section{Statistical testing}
\label{app:statistics}
\paragraph{CMH test.} To compare \oursystem{}s and PEFTs that are the same architecture and parameter count, trained on the same dataset, applied to the same model, and evaluated on a set of $b$ downstream benchmarks, we use the \textbf{Cochran--Mantel--Haenszel} (CMH) test \citep{cochran1954some,mantel1959statistical}.

McNemar's test \citep{mcnemar1947note} is a nonparametric test used to determine whether two binary outcomes (binary correctness on the same benchmark question, in our case) differ in their marginal proportions. CMH is a generalisation of McNemar's test where one has $K$ strata of such binary-outcome paired data (in our case, multiple benchmarks), which controls for stratum before testing for association.

\paragraph{Wilcoxon signed-rank test.} This is a standard nonparametric test for comparing paired differences between two populations, and thus suitable for our prefill vs.~all-position comparisons in aggregate.

\newpage
\section{ReFT without Regret}
\label{app:init}

Two useful properties of LoRA that make learning predictable are: \textbf{zero-delta initialisation} (i.e.~applying an untrained LoRA does not change the model output) and \textbf{learning rate transfer} across ranks. These were not proposed for ReFT previously, so we design an initialisation and parametrisation of both DiReFT and LoReFT that results in these properties, which we use in all our experiments.

\subsection{Zero-delta initialisation} 
\label{app:zero_delta}

The original implementation of ReFT uses default Kaiming uniform initialisation for the weight matrices, and orthonormal Householder parametrisation for $\mathbf{R}$. This means that at initialisation, ReFT applies a non-identity function and thus immediately displaces the model from its original behaviour. Early training is thus wasted on recovering baseline performance.

Like LoRA, we thus seek to set an initialisation such that both DiReFT and LoReFT apply the identity function. To do this, for LoReFT we set $\mathbf{W}^{(0)} := \mathbf{R}^{(0)}$, $\mathbf{b} := \mathbf{0}$, and for DiReFT we set $\mathbf{B}$ to have orthonormal rows at init (but no constraint afterwards), $\mathbf{A} := \mathbf{0}$, $\mathbf{b} := \mathbf{0}$.

\begin{figure}
    \centering
    \includegraphics[width=\linewidth]{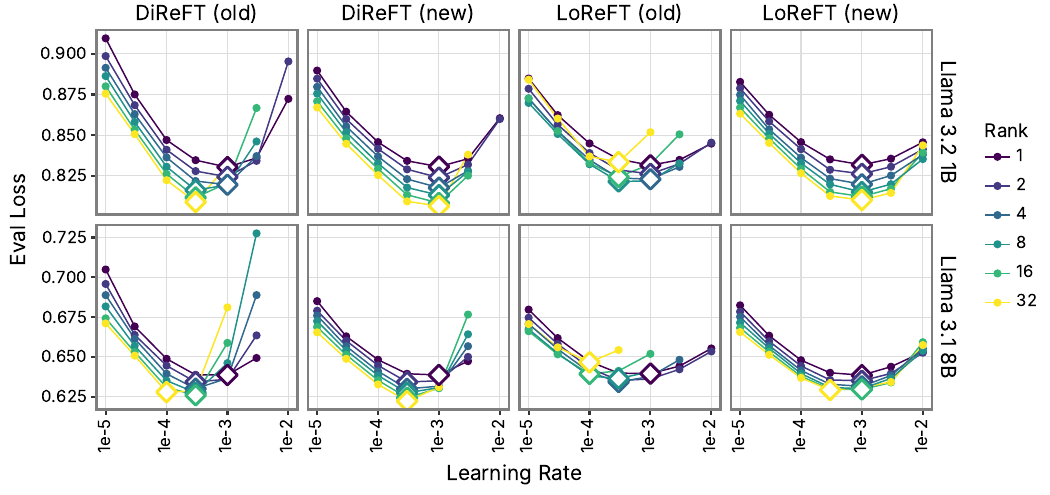}
    \caption{SFT of Llama 3.2 1B Instruct and Llama 3.1 8B Instruct on Tülu-3, comparing old and new parametrisations of DiReFT and LoReFT across varying LR and rank.}
    \label{fig:init-sft}
\end{figure}

\subsection{Learning rate transfer}
\label{app:lr_transfer}
A useful property of LoRA is that initialisation can be set to ensure that the optimal LR (approximately) transfers across ranks $r$, reducing the needed hyperparameter sweeps. This is achieved via applying a scaling prefactor on the LoRA $s_r = \{1, 1/r, 1/\sqrt{r}\}$, chosen depending on details of initialisation.\footnote{For LoRA, \citet{lora} originally use $1/r$ as the scaling prefactor (adopted by \citealp{schulman2025lora}, but with the observation that the LR transfer is imperfect), but later \citet{kalajdzievski2023rankstabilizationscalingfactor} proposes $1/\sqrt{r}$. \citet{chen2026learningratescalinglora} shows that the appropriate prefactor depends on initialisation.} We introduce a scaling prefactor $s_r$ for ReFTs:
\begin{align}
\Phi_{\mathsf{LoReFT}}(\mathbf{h}) &= \mathbf{h} + s_r\mathbf{R}^\top(\mathbf{W}\mathbf{h} + \mathbf{b} - \mathbf{R}\mathbf{h}) \\
\Phi_{\mathsf{DiReFT}}(\mathbf{h}) &= \mathbf{h} + s_r\mathbf{B}^\top(\mathbf{A}\mathbf{h} + \mathbf{b})
\end{align}
To achieve LR transfer across ranks, we propose setting $s_r$ as below:
\begin{theorem}
\label{theorem_init}
Given the following expression for DiReFT, with scaling prefactor $s_r$ being a constant:
\begin{equation}
\delta(\mathbf{h}) = s_r\mathbf{B}^{\top}(\mathbf{A}\mathbf{h} + \mathbf{b})
\end{equation}
to ensure that the $\ell_2$-norm of the delta $\lVert \delta(\mathbf{h}) \rVert$ at the first step of optimisation under Adam is invariant to the DiReFT or LoReFT rank $r$, the scaling prefactor ought to be: $$\boxed{s_r \propto \frac{1}{\sqrt{r}}}$$
\end{theorem}
\begin{proof}
Consider DiReFT at the first optimiser step under Adam. We have:
\begin{equation}
    \delta(\mathbf{h}) = s_r\mathbf{B}^{\top}(\mathbf{A}\mathbf{h} + \mathbf{b})
\end{equation}
where $\mathbf{B}^{(0)} = \mathbf{0}$, $\mathbf{b}^{(0)} = \mathbf{0}$ at initialisation.

Under Adam, given a loss function $\mathcal{L}$ the first update is:
\begin{equation}
\Delta \mathbf{W}^{(0)} = -\eta\frac{\nabla_\mathbf{W}\mathcal{L}}{\lvert \nabla_\mathbf{W}\mathcal{L} \rvert + \epsilon}
\end{equation}
For notation, we now define:
\begin{align}
\mathbf{g} &= \nabla_\delta \mathcal{L} \big|_{\delta = 0}\\
\mathbf{z}^{(i)} &= \mathbf{B}^{(i)}\mathbf{g} \\
\mathbf{a}^{(i)} &= \mathbf{A}^{(i)}\mathbf{h} + \mathbf{b}^{(i)}
\end{align}
Now, after the first optimisation step under Adam, before which $\delta^{(0)} = 0$ due to initialisation, we have:
\begin{align}
\mathbf{A}^{(1)} &= \mathbf{A}^{(0)} - \eta\frac{\nabla_{\mathbf{B}}\mathcal{L}}{\lvert \nabla_{\mathbf{B}}\mathcal{L} \rvert + \epsilon} = -\eta\frac{s_r\mathbf{z}^{(0)}\mathbf{h}^\top}{s_r\lvert \mathbf{z}^{(0)}\mathbf{h}^\top \rvert + \epsilon} \\
\mathbf{b}^{(1)} &= \mathbf{b}^{(0)} - \eta\frac{\nabla_{\mathbf{b}}\mathcal{L}}{\lvert \nabla_{\mathbf{b}}\mathcal{L} \rvert + \epsilon} = -\eta\frac{s_r\mathbf{z}^{(0)}}{s_r\lvert \mathbf{z}^{(0)} \rvert + \epsilon} \\
\mathbf{B}^{(1)} &= \mathbf{B}^{(0)} - \eta\frac{\nabla_{\mathbf{A}}\mathcal{L}}{\lvert \nabla_{\mathbf{A}}\mathcal{L} \rvert + \epsilon} = \mathbf{B}^{(0)}
\end{align}
This means that the delta after the first step is:
\begin{equation}
\delta^{(1)}(\mathbf{h}) = s_r\mathbf{B}^{(0)\top}\left(-\eta\frac{s_r\mathbf{z}^{(0)}\mathbf{h}^\top}{s_r\lvert \mathbf{z}^{(0)}\mathbf{h}^\top \rvert + \epsilon}\mathbf{h} -\eta\frac{s_r\mathbf{z}^{(0)}}{s_r\lvert \mathbf{z}^{(0)} \rvert + \epsilon}\right)
\end{equation}
We want to compute the $\ell_2$-norm of this expression. To make things simpler, we focus on the similarity scores $\mathbf{a}^{(i)} \in \mathbb{R}^r$. Coordinatewise, we have:
\begin{align}
\mathbf{a}_i^{(1)} &= (\mathbf{A}^{(1)}\mathbf{h})_i + \mathbf{b}^{(1)}_i\\
&= \left(\sum_{j=1}^d{\mathbf{A}_{ij}^{(1)}\mathbf{h}_j}\right)  + \mathbf{b}^{(1)}_i \\
&= \left(\sum_{j=1}^d{-\eta\frac{s_r\mathbf{z}^{(0)}_i\mathbf{h}_j^2}{s_r\lvert \mathbf{z}^{(0)}_i\mathbf{h}_j \rvert + \epsilon}}\right) - \eta\frac{s_r\mathbf{z}_i}{s_r\lvert \mathbf{z}_i \rvert + \epsilon} \\
&= -\eta s_r \mathbf{z}_i^{(0)} \left(\left(\sum_{j=1}^d{\frac{\mathbf{h}_j^2}{s_r\lvert \mathbf{z}^{(0)}_i\mathbf{h}_j \rvert + \epsilon}}\right) + \frac{1}{s_r\lvert \mathbf{z}_i \rvert + \epsilon}\right) \\
\end{align}
Now we assume that $\epsilon$ is negligible in the denominators, and the expression simplifies to SignSGD:
\begin{align}
\mathbf{a}_i^{(1)} &\approx -\eta s_r \mathbf{z}_i^{(0)} \left(\left(\sum_{j=1}^d{\frac{\mathbf{h}_j^2}{s_r\lvert \mathbf{z}^{(0)}_i\mathbf{h}_j \rvert}}\right) + \frac{1}{s_r\lvert \mathbf{z}_i \rvert}\right)\\
&= -\eta\sign(\mathbf{z}^{(0)}_i)\left(1 + \sum_{j=1}^d{\lvert \mathbf{h}_j \rvert}\right)\\
&= -\eta\sign(\mathbf{z}^{(0)}_i)\left(1 + \lVert \mathbf{h} \rVert_1\right)
\end{align}
Therefore, the $\ell_2$-norm of $\mathbf{a}^{(1)}$ is:
\begin{equation}
    \lVert \mathbf{a}^{(1)}\rVert_2 = \eta(\lVert \mathbf{h} \rVert_1 + 1)\sqrt{r}
\end{equation}
Using this and the fact that the rows of $\mathbf{B}^{(0)}$ are orthonormal at initialisation (and therefore $\mathbf{B}^{(0)\top}$ is an isometry on $\mathbb{R}^r$), we can compute the $\ell_2$-norm of the delta:
\begin{align}
\lVert \delta^{(1)}(\mathbf{h})\rVert_2 &= s_r\lVert \mathbf{B}^{(0)\top}\mathbf{a}^{(1)}\rVert_2\\
&= s_r \lVert \mathbf{a}^{(1)} \rVert_2\\
&= s_r \eta(\lVert \mathbf{h} \rVert_1 + 1)\sqrt{r}
\end{align}
Therefore, to make the norm of the first delta invariant to rank (modulo the assumption that the Adam epsilon is relatively small), we should set
\begin{equation}
\boxed{s_r \propto \frac{1}{\sqrt{r}}}
\end{equation}
\end{proof}

\paragraph{Validation on SFT.} To confirm that our changes enable optimal LR transfer across ranks for both DiReFT and LoReFT, we use the SFT evaluation from \citet{schulman2025lora}: comparing DiReFT and LoReFT with both original parametrisation and with our new initialisation and scaling prefactor, we finetune Llama 3.2 1B Instruct and Llama 3.1 8B Instruct on the Tülu-3 instruction-tuning dataset \citep{lambert2025tulu3pushingfrontiers}. We intervene on all positions (including decode steps) to ensure a fair comparison to LoRA.\footnote{\citet{wu2024reft} only intervene on some token positions in prefill; this is a useful property of ReFTs but artificially limits their learning capacity, making comparison to LoRA more difficult since it modifies all steps.} We train on $50,000$ examples and evaluate language modelling loss on a held-out set of $500$ samples. We truncate sequences to $2,048$ tokens. We sweep learning rates in $\{10^{-5}, \ldots, 10^{-2}\}$ and ranks in $\{1, 2, 4, 8, 16, 32\}$. We train with the Adam optimiser, with an effective batch size of $32$ (batch size of $2$ with $16$ gradient accumulation steps).

Our results in \cref{fig:init-sft} confirm that the optimal LR near-perfectly transfers across ranks with the new parametrisation ($10^{-3}$ generally, $3 \cdot 10^{-4}$ for DiReFT on 8B) but not the old one; our changes also fix a pathology in the old LoReFT where scaling rank did \textit{not} monotonically improve performance, probably due to a lack of zero-delta initialisation corrupting the model's starting performance.
\newpage
\section{vLLM Integration}
\label{app:vllm}

This appendix expands on Section~\ref{sec:vllm}. We describe how a trained
ReFT adapter reaches a live vLLM worker, how the decoder layer applies the
intervention without breaking \texttt{torch.compile} or CUDA graphs, and how
adapter weights are refreshed during on-policy RL training. The end-to-end
data flow is summarised in Figure~\ref{fig:vllm-arch}.
\input{figs/appendix_vllm_fig}

\subsection{The refactored \texttt{ReftModel} interface}
\label{app:reftmodel}

We refactored the user-facing \texttt{pyreft} API around a single
\texttt{ReftModel} class modelled on PEFT's \texttt{PeftModel}. A
\texttt{ReftModel} wraps a HuggingFace causal-LM, freezes its base weights,
and injects intervention layers at the indices specified by a
\texttt{ReftConfig}. The class exposes the interface a user would expect
from a PEFT adapter---%
\texttt{save\_pretrained} / \texttt{from\_pretrained} for adapter-only
checkpoints, \texttt{print\_trainable\_parameters} for diagnostics,
\texttt{enable\_adapters} / \texttt{disable\_adapters} for ablations. The refactor also
transparently delegates unknown attribute lookups to the wrapped base model
so that training frameworks (HF \texttt{Trainer}, TRL, \texttt{lm-eval})
work unmodified.

Three methods handle the vLLM integration path:
\begin{itemize}
  \item \texttt{export\_vllm\_reft\_spec()} extracts a serialisable blueprint
        of the adapter (class path, constructor kwargs, and CPU state dict)
        plus the set of intervention layer indices and the position schedule.
  \item \texttt{build\_vllm(model\_name, **kwargs)} instantiates a vLLM
        \texttt{LLM} with this blueprint attached to \texttt{VllmConfig}
        and synchronises trained weights into the engine.
  \item \texttt{build\_vllm\_from\_checkpoint(ckpt\_dir, model\_name)} skips
        the HF model entirely, reading only the adapter config and weights
        from disk before handing them to vLLM. This was written to be useful at inference time,
        where holding both an HF copy and a vLLM copy of an 8B+ parameter
        base model wastes GPU memory.
\end{itemize}
Together these three calls are the full public surface area between pyreft
and the vLLM fork; everything else in this appendix is internal.

\subsection{Engine integration}
\label{app:integration}

vLLM's \texttt{VllmConfig} carries per-engine configuration from the host
process to every worker. We add a \texttt{reft\_config} field to this
dataclass that holds the serialised adapter blueprint. At model construction
time, the  causal-LM classes check for a non-empty
\texttt{reft\_config} and, if present, swap their stock decoder-layer class
for a ReFT-aware subclass produced by a factory function. The subclass instantiates an adapter per target
layer from the blueprint, and its \texttt{forward} invokes the base layer
unchanged before adding the intervention delta to the residual stream.
Because the subclass is generated from a shared blueprint rather than
deserialised from a pickled module, worker processes can reconstruct
adapters that contain orthogonal parametrisations (which cannot be pickled
directly) with no special handling.

The engine's weight loader is also patched to mark adapter parameters as
expected-present; otherwise vLLM's checkpoint validator would raise because
the adapter weights are not in the base model's \texttt{safetensors}.

\subsection{Graph-safe execution}
\label{app:graph-safe}

vLLM compiles each decoder layer's forward pass with \texttt{torch.compile}
and wraps it in CUDA graphs for decode. Naively intervening in this path
would either break compilation (dynamic Python branching) or produce stale
results (captured graphs referring to tensors that have since been freed).
We use three design choices to stay compile- and capture-safe.

\paragraph{Adapter inlined, not custom-opped.}
The delta $\Delta = \phi_{\text{adapter}}(h{+}r)$ is computed inline in the
subclassed \texttt{forward}. We did not wrap it in a custom op, because
doing so forced a graph split that cost more than the op saved on simple
LoReFT/DiReFT interventions.

\paragraph{Position mask via pre-computed buffer.}
ReFT applies its delta only at prefill positions (or at all tokens,
depending on the position schedule). Computing the mask inside the compiled
forward would require branching on \texttt{attn\_metadata} fields whose
concrete types differ across attention backends. We instead compute the
mask \emph{outside} the compiled region, in the model runner, before each
forward: the runner reads \texttt{query\_start\_loc} from the attention
metadata to find the prefill/decode split, computes one mask per unique
position schedule, and writes it into a fixed-address \texttt{nn.Buffer}
on each layer (\texttt{\_reft\_position\_mask}). The compiled layer reads
this buffer by slicing, which is a pure tensor op and therefore safe under
capture. Deriving the split from \texttt{query\_start\_loc} rather than
backend-specific fields makes the mask correct under both full and chunked
prefill on every attention kernel vLLM ships. This entire process is skipped when tokens are all marked as decode tokens, speeding up computation.

\paragraph{Graph-safe caches for derived tensors.}
Some adapter quantities are expensive to compute but rarely change; for example, the
Householder product $R$ that materialises the orthogonal rotation for
LoReFT. We compute these once outside the compiled region and store them
in \texttt{nn.Buffer}s at fixed addresses. The compiled forward reads the
buffers directly; no SVD or matrix decomposition ever runs inside a CUDA
graph. When adapter weights change (during RL) these caches are refreshed
by a separate \texttt{collective\_rpc} call (\S\ref{app:weight-sync}).

\subsection{Weight synchronisation for on-policy RL}
\label{app:weight-sync}

On-policy methods such as GRPO require the inference engine's adapter to
track the training adapter step-for-step. Restarting the vLLM engine after
every gradient update is prohibitive, so we expose a
\texttt{collective\_rpc("sync\_reft\_weights", ...)} RPC on the vLLM worker
that accepts a \texttt{\{layer\_idx: state\_dict\}} mapping, loads each
adapter's state dict in place, and refreshes the graph-safe caches of
\S\ref{app:graph-safe}. Because adapter parameters live at stable memory
addresses, the next graph replay sees the new weights without
recapture (similar to how LoRA updates propagate).

\texttt{ReftModel.sync\_to\_vllm(llm)} wraps this RPC on the training side:
it extracts the current adapter state dicts from the HF model, ships them to
the engine, and returns the per-worker sync counts for logging. A typical
RL loop therefore looks like: generate rollouts with the vLLM engine,
compute advantages and take a gradient step on the HF model, call
\texttt{sync\_to\_vllm}, and repeat. Engine startup happens once.

\subsection{\texttt{ReFTRequest}s}
Our forked vLLM adds native support for ReFT adapters as a first-class serving primitive, analogous to how stock vLLM handles LoRA requests. At engine construction time, a \emph{ReFT spec}, which maps transformer layer indices to adapter weights and the intervention position, is registered via \texttt{reft\_config}. At inference time, callers attach a \texttt{ReFTRequest} to each generation call, specifying which adapter to apply; a single vLLM instance can thus serve multiple ReFT adapters by routing different requests to each adapter, just like LoRA serving.

\subsection{Prefill-only LoRA}
Standard vLLM LoRA applies the adapter at every forward pass regardless of generation phase. To support prefill-only LoRA, we extend \texttt{LoRARequest} with a \texttt{lora\_position} field. When \texttt{lora\_position="prefill"} is set, the vLLM LoRA kernel applies the delta only during the prefill phase and bypasses it during decode steps, so generated tokens see only the frozen base-model weights. On the training side (HF/PyTorch), we implement equivalent semantics.

\subsection{Scope and limitations}
\label{app:limits}

The fork currently supports Llama and Qwen2 family decoders, LoRA, DiReFT, and
LoReFT adapters with either a \emph{prefill} or \emph{all-tokens} position
schedule, and tensor-parallel execution on a single node. Speculative
decoding, pipeline parallelism, and non-Llama/Qwen2 architectures have not
been tested and are likely to need additional integration work along the
lines of \S\ref{app:integration}.

\newpage
\section{Benchmarking infrastructure}
\label{app:benchmarking}

This appendix expands on \cref{sec:vllm}. We describe the workload,
the engine configuration, the FCFS scheduling loop, the metrics we report,
and the sweeps that produce each figure in the main text.

\subsection{Workload}
\label{app:bench_workload}

We replicate the request distribution used by Punica
\citep{punica} so that our results sit on the same axes as the
established multi-adapter serving literature.

\paragraph{Length distribution.} Prompt lengths are drawn i.i.d.\ from the following distribution:

$$\mathrm{Lognormal}(\sigma{=}0.8,\,\mathrm{loc}{=}{-}1.0,\,\mathrm{scale}{=}18.0)$$
and clipped to $[1,\,L_{\max}{-}2]$. Each request's total length
(prompt + completion) is drawn uniformly from $[p{+}2,\,L_{\max}]$; the
output length is the difference. We use $L_{\max}=2{,}048$ throughout
unless otherwise noted. Sampling is deterministic for a given seed.

\paragraph{Adapter assignment.} Each request is tagged with an adapter
index drawn from one of four mixes: \emph{identical} (all requests share
adapter $0$), \emph{uniform} (uniform over $N$ adapters; this is our
default), \emph{skewed} (Zipf with $\alpha = 1.0$ over the adapter
catalogue), and \emph{distinct} (request $i$ uses adapter
$i \bmod N$, then shuffled). Most figures report uniform popularity;
the skewed and distinct mixes are used for the popularity-sensitivity
ablations.

\paragraph{Adapter weights.} For LoRA we materialise $N$ random adapters
per layer at the requested rank and target modules; $\mathbf{A}$ and
$\mathbf{B}$ are i.i.d.\ $\mathcal{N}(0,\,0.01^2)$ in fp16. For ReFT we
materialise $N$ adapters of the requested type (\loreft, \direft) by
constructing a CPU adapter at the requested rank and position and
adding $\mathcal{N}(0,\,\sigma_{\text{perturb}}^2)$ noise
($\sigma_{\text{perturb}} = 0.1$) to the learned source weight and bias
to break symmetry across adapters. These weights are not trained:
benchmarks measure serving cost, not task quality.

\paragraph{Prompt content.} Prompt token IDs are drawn uniformly from
$[100,\,32000)$. Sampling is greedy ($T = 0$, $n = 1$) with
\texttt{ignore\_eos = True} and \texttt{max\_tokens} fixed to the
generated output length, so generations always run to the requested
length and timing is not perturbed by stochastic stopping.

\subsection{Engine configuration}
\label{app:bench_engine}

We construct an \texttt{LLMEngine} from \texttt{EngineArgs} with
prefix caching disabled, \texttt{max\_num\_seqs} set to the FCFS
batch size, \texttt{max\_model\_len} $= L_{\max}$, and
\texttt{gpu\_memory\_utilization} $= 0.9$. We never use
\texttt{enforce\_eager} unless explicitly noted, so all numbers in
the main text include CUDA-graph capture.

For LoRA runs we set \texttt{enable\_lora=True},
\texttt{max\_loras=}$M_{\text{active}}$ (the maximum number of distinct
adapters that may be co-batched), \texttt{max\_lora\_rank=}$\max(8,r)$,
and \texttt{max\_cpu\_loras=}$N+1$ so the entire catalogue fits in CPU
memory and only the active set occupies GPU memory at any time. For
ReFT runs we set \texttt{enable\_reft=True}, \texttt{max\_refts=}$M_{\text{active}}$,
and \texttt{max\_cpu\_refts=}$N+1$, then call
\texttt{collective\_rpc("load\_reft\_adapter", \dots)} once per adapter
to push it through the public weight-sync path described in
\cref{app:weight-sync}. For the adapter-less baseline we leave both
features disabled but otherwise reuse the same engine configuration
and the same workload.

Tensor parallelism is set per model: $\text{TP} = 1$ for Qwen2.5-0.5B
and Llama-3.1-8B benchmarks (single H100), and $\text{TP} = 4$ for the
Llama-3.1-70B and Qwen3-30B-A3B benchmarks (four H100s on a single
node). All inference uses bf16 weights.

\subsection{Scheduling loop}
\label{app:bench_loop}

The benchmark drives the engine through \texttt{LLMEngine.step()} rather
than calling \texttt{LLM.generate}, so that we can attribute latency to
the prefill (encode) and decode phases separately. The driver maintains
a \emph{workset} of at most $B$ in-flight requests and admits new
requests one at a time whenever a slot frees up; admission is
first-come-first-served over the synthetic stream. For each in-flight
request we record:
\begin{itemize}
  \item its \emph{encode latency} as the wall time between
        \texttt{add\_request} and the step in which the first output
        token is observed;
  \item its \emph{decode latency} as the wall time between the first
        output token and the step that marks the request finished;
  \item its prompt and output lengths (from the workload, not from
        engine state).
\end{itemize}
The engine's own scheduler is responsible for prefill/decode chunking
and continuous batching --- we do not change vLLM's scheduling policy.
A separate warmup run with the same workload (different seed) is
executed before every timed run and discarded; this absorbs CUDA-graph
capture, weight loading, and any one-off allocator behaviour so the
timed run reflects steady-state serving cost.

\subsection{Metrics}
\label{app:bench_metrics}

We report four quantities per run:
\begin{itemize}
  \item \textbf{Per-token encode latency}, computed per request as
        encode latency divided by prompt length, then aggregated as the
        $p_{50}$, $p_{90}$, $p_{99}$, mean, and standard deviation
        across requests.
  \item \textbf{Per-token decode latency}, defined analogously over
        output tokens.
  \item \textbf{End-to-end throughput}, the total number of prompt and
        output tokens served divided by the total wall time of the
        timed run.
  \item \textbf{Per-request encode and decode latency}, with the same
        percentile breakdown, useful for tail-latency comparisons.
\end{itemize}
Unless otherwise noted, plots in the main text show median per-token
latency with shaded $p_{10}$--$p_{90}$ bands and median throughput.

\newpage

\section{LongWriter evaluation}
\label{app:longwriter-scoring}

We use the LongBench-Write benchmark of \citet{longwriter}: 120 open-ended
writing prompts paired with a target output length. We partition results into five
required-length buckets ($<\!500$, $500$--$2$k, $2$k--$4$k, $4$k--$20$k,
$20$k$+$ words) to gauge performance on each of the buckets. The headline metrics are a length-following score
$S_l \in [0, 100]$ and a quality score $S_q \in [0, 100]$, each averaged over
the prompts in a bucket. We additionally report the fraction of generations
truncated at the decoding cap (\cref{sec:trunc}) to disentangle ``can't write
more'' from ``won't write more'' failure modes.

\paragraph{Generation.}
For every prompt $i$ with target length $\ell_i^{\text{req}}$, we sample a
single response $y_i$ at temperature $0.5$ with a 32{,}768-token output cap.
We measure produced length $\ell_i^{\text{gen}}$ in words using the
LongWriter convention: each Chinese ideograph contributes one ``word'' and
each maximal run of $[\text{a-zA-Z}]$ contributes one English word; whitespace
and punctuation are discarded.

\paragraph{Length-following score $S_l$.}
We use the piecewise score from the LongWriter paper unchanged:
\begin{equation}
S_l(\ell^{\text{req}}, \ell^{\text{gen}}) =
\begin{cases}
100 \cdot \max\!\Big(0,\, 1 - \tfrac{1}{3}\big(\tfrac{\ell^{\text{gen}}}{\ell^{\text{req}}} - 1\big)\Big)
  & \text{if } \ell^{\text{gen}} > \ell^{\text{req}} \\[4pt]
100 \cdot \max\!\Big(0,\, 1 - \tfrac{1}{2}\big(\tfrac{\ell^{\text{req}}}{\ell^{\text{gen}}} - 1\big)\Big)
  & \text{if } 0 < \ell^{\text{gen}} \le \ell^{\text{req}} \\[4pt]
0 & \text{otherwise.}
\end{cases}
\label{eq:sl}
\end{equation}
The asymmetry penalises under-shooting twice as harshly as over-shooting (a
4$\times$ overshoot reaches $S_l = 0$, while a 3$\times$ undershoot does).
Bucket-level $S_l$ is the unweighted mean over prompts whose required length
falls in that bucket.

\paragraph{Quality score $S_q$.}
We follow the LLM-as-judge protocol of \citet{longwriter}, replacing
their GPT-4 judge with \texttt{gpt-5-mini} with low reasoning for cost. For each (prompt, response)
pair the judge returns six integer scores in $\{1, \dots, 5\}$ along the
LongWriter rubric:
\textit{Relevance}, \textit{Accuracy}, \textit{Coherence}, \textit{Clarity},
\textit{Breadth and Depth}, and \textit{Reading Experience}. We use the
LongWriter judge prompt verbatim, including the instruction \textit{not} to
penalise length (so $S_l$ and $S_q$ measure orthogonal failure modes). The
judge is instructed to emit JSON; we parse and clip each dimension to
$[1, 5]$, dropping items whose JSON cannot be recovered. The per-item raw
score $\bar s_i \in [1, 5]$ is the unweighted mean across the six dimensions
(over those that parsed). We report two equivalent forms:
\[
S_q^{\text{raw}} = \frac{1}{|B|}\sum_{i \in B} \bar s_i \in [1, 5],
\qquad
S_q = \frac{S_q^{\text{raw}} - 1}{4} \cdot 100 \in [0, 100],
\]
where $B$ is the set of prompts in a given bucket. Throughout the main paper
we plot the $0$--$100$ form so $S_l$ and $S_q$ share a common scale.

\paragraph{Truncation rate.}
\label{sec:trunc}
Every generation reports a finish reason from the decoding backend
(\texttt{stop} when an EOS / stop sequence is emitted, \texttt{length} when
the 32{,}768-token cap is hit). The bucket-level truncation rate is
\[
\tau_B = \frac{|\{ i \in B : \text{finish\_reason}_i = \texttt{length} \}|}{|B|} \in [0, 1],
\]
which we report as a percentage. A response that ends voluntarily inside the
budget contributes $0$ to $\tau_B$; a response that would have continued past
$32$k tokens contributes $1$. Truncation rate complements
$S_l$ by separating two regimes: a low $S_l$ paired with low $\tau_B$ implies
the model ended early on its own; a low $S_l$ paired with high $\tau_B$
implies the model wanted to continue but ran out of decoding budget.

\newpage

\section{Comparing SFT and RL on GSM8K}
\label{sec:gsm8k-more}

\input{tabs/sft_vs_rl_r8}

To investigate the gap on GSM8K under RL further, we conduct an ablation with identical training data under both SFT and RLVR, comparing \oursystem{}s against their all-position counterparts. \cref{tab:sft_vs_rl_r8} reports final GSM8K accuracy for each (method, training regime) pair, with $\Delta$ columns showing the prefill-vs-all-position gap within each regime. In all cases, \oursystem{}s underperform PEFTs by some margin. These results suggest that the restriction to prefill positions is particularly difficult for learning on GSM8K. This also agrees with prior comparisons on math reasoning for ReFT \citep{wu2024reft}; however, this phenomenon is not replicated in MATH, which is a more difficult benchmark. We leave further investigation of this gap to future work.

\newpage

\section{Long output attenuation experiments}
\label{app:output_attenuation}

\begin{figure}[!t]
    \centering
        \includegraphics[width=\linewidth]{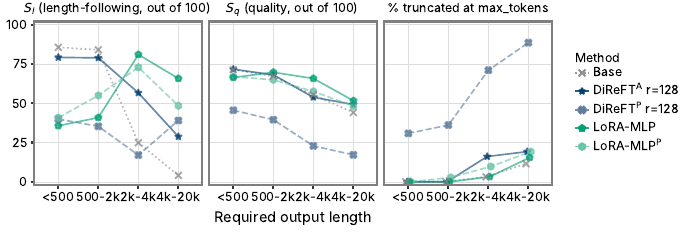}
    \caption{\textbf{\lorap{} on MLP-only still outperforms \direftp{} and \direfta{} at high ranks}
    LongWriter on Llama-3.1-8B-Instruct (rank $16$), four required
    output-length brackets.}
    \label{fig:longwriter_results_mlp}
\end{figure}

In the LongBench-Write experiments, we find that \direftp{} runs fail
while \lorap{} runs do not. We conduct two follow-up experiments to
investigate why.

\paragraph{Is the failure mediated by attention?}
Our initial hypothesis was that the asymmetry traces to \emph{where}
each adapter writes during prefill. LoRA on attention projections
modifies the K and V vectors written to the cache, so every subsequent
decode token attends to KV-cache contents directly shaped by the
trained delta. DiReFT instead writes additively to the residual stream
at prompt positions and reaches the cache only indirectly through the
un-modified K/V projections. If this distinction is what allows
\lorap{} to survive decode, then restricting LoRA's target modules
to MLP only (\texttt{gate\_proj}, \texttt{up\_proj},
\texttt{down\_proj}) should reproduce \direftp{}'s failure mode ---
neither variant would have direct access to K/V.

\cref{fig:longwriter_results_mlp} refutes this hypothesis. MLP-only
\lorap{} performs on par with MLP-only LoRA and high rank \direftp{} and \direfta{} at every length bracket
on both $S_l$ and $S_q$, and avoids the runaway-truncation regime
that \direftp{} occupies. Removing attention-projection LoRA does
not break the prefill-only setting.

\paragraph{Is the failure due to under-parameterisation?}
A second possibility is that \direftp{} simply lacks the capacity at
rank 16 to encode the conditional length-matching signal LoRA manages
at the same rank, given LoRA's larger effective parameter count across
seven target modules. We re-train \direftp{} at rank 128 and
re-evaluate. The overshoot pathology persists: \direftp{} r=128
produces $\geq$7k-word generations across every length bracket and
saturates the 32k decoding cap on more than 70\% of $\geq$2k requests.
Increased capacity does not recover the prefill-only setting.

\paragraph{Implication.}
Neither restricting LoRA to MLP modules nor scaling DiReFT to rank 128
changes the asymmetry observed in the main results. The difference
between \lorap{} and \direftp{} appears not to depend on which
projections each adapter targets, nor on rank. We conjecture the
failure is structural: LoRA's multiplicative perturbation reshapes
how prompt tokens are encoded into hidden states, and that
reshaping propagates through the KV cache to persist through decode;
DiReFT's additive residual write at prompt positions injects a
fixed signal that fails to translate into the conditional,
target-dependent decoding behaviour LongWriter requires. A precise
mechanistic characterisation is left to future work.
 
\newpage
\input{sections/opd}
\newpage
\section{Broader Societal Impacts}
\label{app:society}
\paragraph{Positive societal impacts.} PreFT improves the inference efficiency of multi-tenant adapter serving for personalized large language models, with throughput gains of up to $1.9\times$ at scale on Llama-3.1-70B. The most direct positive impact is a reduction in energy and compute costs per served request, which both lowers the environmental footprint of personalized LLM deployment and reduces the economic barrier to offering personalized applications. Cheaper multi-tenant serving particularly benefits use cases where per-user adaptation matters but serving costs have been prohibitive, including educational tools adapted to individual learners, accessibility applications tailored to specific users' communication patterns, and language- or dialect-specific adaptations for communities under-represented in base model training data. The per-adapter serving paradigm that PreFT supports is also structurally compatible with privacy-preserving personalization architectures, where user-specific weights remain isolated rather than aggregated into a shared model or passed through every inference as long-context input. By releasing our vLLM fork and refactored pyreft library, we additionally lower the engineering bar for academic and independent researchers studying multi-adapter serving, on-policy reinforcement learning with adapters, and related topics that previously required substantial systems infrastructure to investigate.

\paragraph{Negative societal impacts.} The negative impacts of PreFT are not specific to our method but apply to efficiency improvements in LLM serving generally. Lowering the cost of personalized LLM deployment increases the economic feasibility of harmful applications alongside beneficial ones, including automated disinformation, harassment, and manipulative personalization. PreFT does not introduce new model capabilities; it makes existing capabilities more efficient to serve. Concerns about misuse of personalized LLMs therefore apply to the underlying base models and to deployment decisions made by operators, rather than to our methodological contribution. We release a serving framework and adapter infrastructure; we do not release new model weights, datasets, or end-user applications that would amplify these concerns. We also note the standard Jevons-paradox observation that efficiency improvements tend to increase aggregate usage of the underlying technology, so even efficiency-driven reductions in per-request energy may be offset by increased deployment volume; this is a property of efficiency contributions in general and is not specific to PreFT.
\newpage

\input{tabs/sft_lr_rank_tables}

\newpage

\newpage
\section{Asset licenses and attribution}
\label{app:licenses}

We use the following datasets, models, and software, listed with their licenses. All assets are used in accordance with their license terms.

\paragraph{Datasets.}
\begin{itemize}
    \item \textbf{T\"ulu-3} \citep{lambert2025tulu3pushingfrontiers}: ODC-BY-1.0 (with subset-specific variations; some subsets are non-commercial). \url{https://huggingface.co/datasets/allenai/tulu-3-sft-mixture}
    \item \textbf{OpenThoughts} \citep{guha2025openthoughtsdatarecipesreasoning}: Apache 2.0. \url{https://huggingface.co/datasets/open-thoughts/OpenThoughts-114k}
    \item \textbf{GSM8K} \citep{gsm8k}: MIT.
    \item \textbf{MATH} \citep{math}: MIT.
    \item \textbf{MBPP} \citep{mbpp}: CC-BY 4.0.
    \item \textbf{HumanEval} \citep{humaneval}: MIT.
    \item \textbf{LongBench-Write} \citep{longwriter}: Apache 2.0.
\end{itemize}

\paragraph{Base models.}
\begin{itemize}
    \item \textbf{Llama-3.1, Llama-3.2}: Llama 3.x Community License (Meta).
    \item \textbf{Qwen-2.5}: Tongyi Qianwen License (Alibaba).
\end{itemize}

\paragraph{Software.}
\begin{itemize}
    \item \textbf{vLLM} \citep{vllm}: Apache 2.0. We extend vLLM with multi-ReFT serving support; our fork inherits the Apache 2.0 license.
    \item \textbf{pyreft} \citep{wu2024reft}: Apache 2.0. We refactor pyreft to support a unified \texttt{ReftModel} interface and vLLM integration; our refactor inherits the Apache 2.0 license.
    \item \textbf{trl, peft, transformers}: Apache 2.0 (Hugging Face).
\end{itemize}

\paragraph{LLM-as-judge.}
We use OpenAI's \texttt{gpt-5-mini} model via the OpenAI API as the judge model for LongBench-Write quality scoring (\cref{app:longwriter-scoring}), in accordance with OpenAI's API terms of service. The judge is used only for evaluation; no training data is generated using this model.


\end{document}

%% file: tabs/sft_prefill_vs_allpos.tex
\begin{table}[t]
\centering
\caption{\textbf{PEFTs and \oursystem{}s have no significant difference on downstream performance under SFT.} SFT of base Llama models, evaluated on T\"ulu-3, gives similar downstream scores for parameter-matched prefill vs.~all-position adapters. \texttt{*} indicates statistical significance.}
\label{tab:sft_prefill_vs_allpos}
\begin{tabular}{l c ccc ccc}
\toprule
 & & \multicolumn{2}{c}{\emph{Llama-3.2-1B (base)}} & & \multicolumn{2}{c}{\emph{Llama-3.1-8B (base)}} & \\
\cmidrule(lr){3-4} \cmidrule(lr){6-7}
\textbf{Method} & $r$ & Prefill & All-pos & $\Delta$ & Prefill & All-pos & $\Delta$ \\
\midrule
\multicolumn{2}{l}{\emph{Base (no adapter)}} & \multicolumn{2}{c}{$10.89$} & & \multicolumn{2}{c}{$13.85$} & \\
\rowcolor{gray!12} \quad DiReFT & $1$ & $14.13$ & $16.67$ & \textcolor{red!70!black}{\texttt{-2.53}\textsuperscript{*}} & $48.85$ & $46.41$ & \textcolor{green!50!black}{\texttt{+2.44}\phantom{\textsuperscript{*}}} \\
\quad DiReFT & $4$ & $16.66$ & $16.48$ & \textcolor{green!50!black}{\texttt{+0.18}\phantom{\textsuperscript{*}}} & $48.18$ & $50.60$ & \textcolor{red!70!black}{\texttt{-2.42}\textsuperscript{*}} \\
\rowcolor{gray!12} \quad LoRA & $1$ & $16.60$ & $17.70$ & \textcolor{red!70!black}{\texttt{-1.10}\textsuperscript{*}} & $46.46$ & $49.58$ & \textcolor{red!70!black}{\texttt{-3.12}\textsuperscript{*}} \\
\rowcolor{gray!12} \quad DiReFT & $16$ & $17.77$ & $17.50$ & \textcolor{green!50!black}{\texttt{+0.27}\phantom{\textsuperscript{*}}} & $49.29$ & $52.72$ & \textcolor{red!70!black}{\texttt{-3.43}\textsuperscript{*}} \\
\quad LoRA & $4$ & $19.31$ & $18.99$ & \textcolor{green!50!black}{\texttt{+0.32}\phantom{\textsuperscript{*}}} & $47.56$ & $48.87$ & \textcolor{red!70!black}{\texttt{-1.31}\textsuperscript{*}} \\
\quad DiReFT & $64$ & $17.96$ & $18.60$ & \textcolor{red!70!black}{\texttt{-0.64}\textsuperscript{*}} & $50.61$ & $52.25$ & \textcolor{red!70!black}{\texttt{-1.65}\textsuperscript{*}} \\
\rowcolor{gray!12} \quad LoRA & $16$ & $18.88$ & $18.14$ & \textcolor{green!50!black}{\texttt{+0.74}\textsuperscript{*}} & $49.06$ & $51.31$ & \textcolor{red!70!black}{\texttt{-2.25}\phantom{\textsuperscript{*}}} \\
\quad LoRA & $64$ & $19.82$ & $19.63$ & \textcolor{green!50!black}{\texttt{+0.18}\phantom{\textsuperscript{*}}} & $49.55$ & $49.66$ & \textcolor{red!70!black}{\texttt{-0.12}\phantom{\textsuperscript{*}}} \\
\multicolumn{2}{l}{\emph{Instruct (ref.)}} & \multicolumn{2}{c}{$39.57$} & & \multicolumn{2}{c}{$71.28$} & \\
\bottomrule
\end{tabular}
\end{table}

%% file: tabs/rl_new.tex
\begin{table}[t]
\centering
\caption{\textbf{\oursystem{}s nearly match PEFTs on downstream performance under RL.} We report prefill vs.~all-pos performance of both DiReFT and LoRA on three RL tasks, along with deltas. We find that \oursystem{}s approach PEFT performance on MATH and HumanEval but lag on GSM8K. \texttt{*} indicates statistical significance.}
\label{tab:rl_new_all_tasks_by_rank}
\begin{adjustbox}{max width=\textwidth}
\begin{tabular}{l c cc c cc c cc c}
\toprule
 &  & \multicolumn{2}{c}{\emph{GSM8K}} &  & \multicolumn{2}{c}{\emph{MATH}} &  & \multicolumn{2}{c}{\emph{HumanEval}} &  \\
\cmidrule(lr){3-4} \cmidrule(lr){6-7} \cmidrule(lr){9-10}
\textbf{Method} & $r$ & Prefill & All-pos & $\Delta_{\text{G}}$ & Prefill & All-pos & $\Delta_{\text{M}}$ & Prefill & All-pos & $\Delta_{\text{H}}$ \\
\midrule
\multicolumn{2}{l}{\emph{Llama-3.1-8B (base)}} & \multicolumn{2}{c}{$24.79$} &  & \multicolumn{2}{c}{$17.14$} &  & \multicolumn{2}{c}{$37.20$} &  \\
\rowcolor{gray!12} \quad DiReFT & $1$ & $59.44$ & $57.09$ & \textcolor{green!50!black}{\texttt{+2.35}\phantom{\textsuperscript{*}}} & $12.40$ & $13.48$ & \textcolor{red!70!black}{\texttt{-1.08}\textsuperscript{*}} & $40.24$ & $40.85$ & \textcolor{red!70!black}{\texttt{-0.61}\phantom{\textsuperscript{*}}} \\
\quad DiReFT & $8$ & $60.12$ & $66.94$ & \textcolor{red!70!black}{\texttt{-6.82}\textsuperscript{*}} & $14.12$ & $17.56$ & \textcolor{red!70!black}{\texttt{-3.44}\textsuperscript{*}} & $40.85$ & $28.05$ & \textcolor{green!50!black}{\texttt{+12.80}\textsuperscript{*}} \\
\rowcolor{gray!12} \quad LoRA & $1$ & $58.76$ & $67.02$ & \textcolor{red!70!black}{\texttt{-8.26}\textsuperscript{*}} & $14.62$ & $14.92$ & \textcolor{red!70!black}{\texttt{-0.30}\phantom{\textsuperscript{*}}} & $37.80$ & $37.20$ & \textcolor{green!50!black}{\texttt{+0.60}\phantom{\textsuperscript{*}}} \\
\quad DiReFT & $32$ & $64.22$ & $68.16$ & \textcolor{red!70!black}{\texttt{-3.94}\textsuperscript{*}} & $10.96$ & $17.82$ & \textcolor{red!70!black}{\texttt{-6.86}\textsuperscript{*}} & $44.51$ & $51.22$ & \textcolor{red!70!black}{\texttt{-6.71}\textsuperscript{*}} \\
\rowcolor{gray!12} \quad LoRA & $8$ & $62.24$ & $53.45$ & \textcolor{green!50!black}{\texttt{+8.79}\textsuperscript{*}} & $13.94$ & $17.08$ & \textcolor{red!70!black}{\texttt{-3.14}\textsuperscript{*}} & $41.46$ & $50.00$ & \textcolor{red!70!black}{\texttt{-8.54}\textsuperscript{*}} \\
\quad LoRA & $32$ & $60.27$ & $64.67$ & \textcolor{red!70!black}{\texttt{-4.40}\textsuperscript{*}} & $16.66$ & $16.52$ & \textcolor{green!50!black}{\texttt{+0.14}\phantom{\textsuperscript{*}}} & $34.76$ & $25.61$ & \textcolor{green!50!black}{\texttt{+9.15}\textsuperscript{*}} \\
\midrule
\multicolumn{2}{l}{\emph{Qwen-2.5-0.5B (base)}} & \multicolumn{2}{c}{$13.50$} &  & \multicolumn{2}{c}{$26.00$} &  & \multicolumn{2}{c}{$10.37$} &  \\
\rowcolor{gray!12} \quad DiReFT & $1$ & $48.14$ & $49.28$ & \textcolor{red!70!black}{\texttt{-1.14}\phantom{\textsuperscript{*}}} & $22.14$ & $32.42$ & \textcolor{red!70!black}{\texttt{-10.28}\textsuperscript{*}} & $11.59$ & $9.76$ & \textcolor{green!50!black}{\texttt{+1.83}\phantom{\textsuperscript{*}}} \\
\quad DiReFT & $8$ & $49.51$ & $52.69$ & \textcolor{red!70!black}{\texttt{-3.18}\textsuperscript{*}} & $31.34$ & $30.02$ & \textcolor{green!50!black}{\texttt{+1.32}\textsuperscript{*}} & $15.85$ & $11.59$ & \textcolor{green!50!black}{\texttt{+4.26}\phantom{\textsuperscript{*}}} \\
\rowcolor{gray!12} \quad LoRA & $1$ & $45.56$ & $50.87$ & \textcolor{red!70!black}{\texttt{-5.31}\textsuperscript{*}} & $26.88$ & $31.00$ & \textcolor{red!70!black}{\texttt{-4.12}\phantom{\textsuperscript{*}}} & $18.90$ & $23.78$ & \textcolor{red!70!black}{\texttt{-4.88}\phantom{\textsuperscript{*}}} \\
\quad DiReFT & $32$ & $48.07$ & $48.75$ & \textcolor{red!70!black}{\texttt{-0.68}\phantom{\textsuperscript{*}}} & $32.68$ & $31.86$ & \textcolor{green!50!black}{\texttt{+0.82}\phantom{\textsuperscript{*}}} & $10.37$ & $15.24$ & \textcolor{red!70!black}{\texttt{-4.87}\phantom{\textsuperscript{*}}} \\
\rowcolor{gray!12} \quad LoRA & $8$ & $43.97$ & $52.01$ & \textcolor{red!70!black}{\texttt{-8.04}\textsuperscript{*}} & $25.30$ & $31.84$ & \textcolor{red!70!black}{\texttt{-6.54}\textsuperscript{*}} & $10.98$ & $10.98$ & \textcolor{green!50!black}{\texttt{+0.00}\phantom{\textsuperscript{*}}} \\
\quad LoRA & $32$ & $40.94$ & $49.73$ & \textcolor{red!70!black}{\texttt{-8.79}\textsuperscript{*}} & $33.18$ & $32.96$ & \textcolor{green!50!black}{\texttt{+0.22}\phantom{\textsuperscript{*}}} & $21.95$ & $25.61$ & \textcolor{red!70!black}{\texttt{-3.66}\phantom{\textsuperscript{*}}} \\
\bottomrule
\end{tabular}
\end{adjustbox}
\end{table}

%% file: figs/appendix_vllm_fig.tex
\begin{figure}[t]
    \centering
    \begin{tikzpicture}[
        font=\small,
        >={Stealth[length=4.5pt]},
        every node/.style={align=center},
        colbox/.style={
            draw, rounded corners=2pt,
            minimum width=42mm, minimum height=10mm,
            inner sep=3pt,
        },
        loadbox/.style={colbox, fill=blue!6},
        syncbox/.style={colbox, fill=orange!10},
        wide/.style={
            draw, rounded corners=2pt,
            minimum width=96mm, minimum height=12mm,
            inner sep=4pt,
        },
        layerbox/.style={wide, fill=green!8},
        runnerbox/.style={wide, fill=violet!6},
        lbl/.style={font=\scriptsize, inner sep=1.5pt},
        arr/.style={->, thick},
        darr/.style={->, thick, dashed},
    ]

    \def\colsep{8mm}
    \def\rowsep{7mm}

    \node[loadbox] (ckpt)
        {Trained \texttt{ReftModel}\\\footnotesize adapter weights + config};
    \node[loadbox, below=\rowsep of ckpt] (build)
        {\texttt{ReftModel.build\_vllm(...)}};
    \node[loadbox, below=\rowsep of build] (cfg)
        {\texttt{VllmConfig.reft\_config}\\\footnotesize serialized blueprint};

    \node[syncbox, right=\colsep of ckpt]  (trainer)
        {RL / SFT trainer\\\footnotesize live adapter updates};
    \node[syncbox, right=\colsep of build] (sync)
        {\texttt{sync\_to\_vllm(llm)}};
    \node[syncbox, right=\colsep of cfg]   (rpc)
        {\texttt{collective\_rpc(}\\\texttt{\textquotesingle sync\_reft\_weights\textquotesingle\texttt{)}}};

    \coordinate (midx) at ($(cfg.east)!0.5!(rpc.west)$);
    \node[layerbox, below=12mm of midx, anchor=north] (layer)
        {ReFT-aware decoder layer\\[1pt]
         \footnotesize $\Delta \,=\, \phi_{\text{adapter}}(h{+}r)$, masked by
         \texttt{\_reft\_position\_mask}};
    \node[runnerbox, below=\rowsep of layer] (runner)
        {Model runner \,\,{\footnotesize\itshape(outside compiled region)}\\[1pt]
         \footnotesize position mask $\leftarrow$ \texttt{query\_start\_loc}};

    \draw[arr] (ckpt)    -- (build)   node[midway, lbl, fill=white] {load};
    \draw[arr] (build)   -- (cfg);
    \draw[arr] (trainer) -- (sync)    node[midway, lbl, fill=white] {update};
    \draw[arr] (sync)    -- (rpc);

    \draw[arr] (cfg.south) -- ($(layer.north)+(-22mm,0)$)
        node[midway, lbl, fill=white] {install adapter};
    \draw[arr] (rpc.south) -- ($(layer.north)+(+22mm,0)$)
        node[midway, lbl, fill=white] {in-place \texttt{copy\_}};

    \draw[darr] (runner.north) -- (layer.south)
        node[midway, lbl, fill=white] {mask buffer};

    \begin{scope}[on background layer]
        \node[draw, dotted, rounded corners=3pt,
              inner xsep=5pt, inner ysep=5pt,
              fit=(layer)(runner)] (worker) {};
    \end{scope}
    \node[lbl, anchor=north east, font=\scriptsize\itshape]
        at ($(worker.north east)+(-5pt,8pt)$) {vLLM worker};

    \end{tikzpicture}
    \caption{Architecture of the \oursystem{} fork. The left column is the
    one-time load path from a trained \texttt{ReftModel} into the engine; the
    right column is the per-step training--inference sync path used during
    on-policy RL. Solid arrows carry adapter configuration or weights; the
    dashed arrow marks per-forward state (the position mask) written by the
    model runner outside the compiled region. Adapter parameters live at
    stable memory addresses inside the compiled decoder layer, so in-place
    updates propagate through subsequent graph replays without rebuilding
    the engine.}
    \label{fig:vllm-arch}
\end{figure}

%% file: tabs/sft_vs_rl_r8.tex
\begin{table}[t]
\centering
\caption{\textbf{\oursystem{}s lag PEFTs on GSM8K under both SFT and RL.} GSM8K test accuracy on \texttt{Llama-3.1-8B} with rank-$8$ adapters; the best learning rate was selected for each (training, method, position) cell. \texttt{*} indicates statistical significance.}
\label{tab:sft_vs_rl_r8}
\begin{adjustbox}{max width=\textwidth}
\begin{tabular}{l c cc c cc c}
\toprule
 &  & \multicolumn{2}{c}{\emph{SFT}} &  & \multicolumn{2}{c}{\emph{RL}} &  \\
\cmidrule(lr){3-4} \cmidrule(lr){6-7}
\textbf{Method} & $r$ & Prefill & All-pos & $\Delta_{\text{SFT}}$ & Prefill & All-pos & $\Delta_{\text{RL}}$ \\
\midrule
\multicolumn{2}{l}{\emph{Llama-3.1-8B (base)}} & \multicolumn{2}{c}{$24.79$} &  & \multicolumn{2}{c}{$24.79$} &  \\
\rowcolor{gray!12} \quad LoRA & $8$ & $56.10$ & $60.80$ & \textcolor{red!70!black}{\texttt{-4.70}\phantom{\textsuperscript{*}}} & $62.24$ & $73.31$ & \textcolor{red!70!black}{\texttt{-11.07}\textsuperscript{*}} \\
\quad DiReFT & $8$ & $47.69$ & $57.24$ & \textcolor{red!70!black}{\texttt{-9.55}\textsuperscript{*}} & $60.12$ & $66.94$ & \textcolor{red!70!black}{\texttt{-6.82}\textsuperscript{*}} \\
\bottomrule
\end{tabular}
\end{adjustbox}
\end{table}

%% file: sections/opd.tex
\section{On-policy distillation}
\label{sec:opd}

We find a gap between \oursystem{}s and all-position PEFTs when training with teacher-forced SFT, but no such gap exists in our RL experiments. We hypothesised that on-policy training being more effective for prefill-only adapters implies that we may outperform SFT if we do \textbf{on-policy distillation} \citep{opd,lu2025onpolicydistillation} from an all-positions teacher adapter into a \oursystem{}.

\paragraph{Implementation.}  We take the GRPO implementation in \texttt{trl} and override the reference model (which is used to compute the KL term in GRPO) to apply the teacher adapter. We pass dummy rewards of $0$, so that the loss consists of only the KL term. We set $\beta := 1$ and sweep learning rates.

\paragraph{Experiment 1: Distilling rank-$4$ LoRA on Llama 3.2 1B Instruct.} We take the best-performing rank-$4$ all-positions LoRA finetuned on Llama 3.2 1B Instruct on Tülu-3, and distill it into a prefill-only LoRA.
We show results below for various LRs after $500$ steps of training. Lower KL loss (which is the distillation target) does correlate with lower evaluation loss on Tülu, indicating the distillation is successfully transferring in-distribution learning from the teacher to the student.

Unfortunately, lower learning rates achieve worse distillation loss but better downstream performance, indicating that the training is harming out-of-distribution evaluations. However, we later realised that the Llama Instruct models are generally too strong at instruction-following tasks for Tülu-finetuning to matter, so we instead switch experiments to the base model. Our teacher is actually worse than the untrained student!

\begin{table}[!h]
\centering
\label{tab:distill-prefill-lora}
\begin{tabular}{l cc ccc}
\toprule
& \multicolumn{2}{c}{Loss ($\downarrow$)} & \multicolumn{3}{c}{Downstream Evals ($\uparrow$)} \\
\cmidrule(lr){2-3} \cmidrule(lr){4-6}
$\eta$ & KL & Tülu eval. & GSM8K & MMLU & IFEval \\
\midrule
\textit{(teacher)} & --- & $0.812$ & $34.2$ & $27.6$ & $40.3$ \\
\midrule
\textit{(untrained student)} & --- & $0.963$ & \underline{$36.9$} & $36.9$ & $44.9$ \\
$1 \cdot 10^{-3}$ & $0.075$ & $0.941$ & $20.8$ & $36.6$ & $39.7$ \\
$3 \cdot 10^{-4}$ & \underline{$0.013$} & \underline{$0.939$} & $21.8$ & $32.2$ & $39.2$ \\
$1 \cdot 10^{-4}$ & $0.017$ & $0.947$ & $20.4$ & $35.1$ & $41.8$ \\
$3 \cdot 10^{-5}$ & $0.022$ & $0.958$ & $23.4$ & $37.0$ & $42.9$ \\
$1 \cdot 10^{-5}$ & $0.027$ & $0.962$ & $25.8$ & \underline{$37.0$} & \underline{$45.3$} \\
\bottomrule
\end{tabular}
\end{table}

\paragraph{Experiment 2: Distilling rank-$1$ LoRA on Llama 3.1 1B (base).} Given that finetuning an already instruction-tuned model on Tülu-3 is not informative (i.e., the teacher is not better than the student at the start of training), we instead perform on-policy distillation from a teacher adapter train on the \textit{base} model, as in the main text (\cref{sec:results}). We take the best all-positions rank-$1$ LoRA on Llama 3.2 1B base and on-policy distill it into \lorap{}.

Below, we present our results. On-policy distillation does lift MMLU performance from below-chance to slightly above chance, very slightly improves GSM8K, but actually hurts IFEval. This leads us to conclude that on-policy distillation is not particularly promising in this setting; we may require a better model to experiment on, better teachers, longer training, or different training/evaluation sets, all of which are beyond our computational budget.

\begin{table}[!h]
\centering
\label{tab:distill-prefill-lora-base}
\begin{tabular}{l cc ccc}
\toprule
& \multicolumn{2}{c}{Loss ($\downarrow$)} & \multicolumn{3}{c}{Downstream Evals ($\uparrow$)} \\
\cmidrule(lr){2-3} \cmidrule(lr){4-6}
$\eta$ & KL & Tülu eval. & GSM8K & MMLU & IFEval \\
\midrule
\textit{(teacher)} & --- & $0.901$ & $8.0$ & $23.9$ & $21.6$ \\
\midrule
\textit{(untrained student)} & --- & $1.106$ & $3.0$ & $13.3$ & \underline{$16.5$} \\
$1 \cdot 10^{-3}$ & $0.074$ & $1.121$ & $3.5$ & $26.2$ & $9.4$ \\
$3 \cdot 10^{-4}$ & $0.030$ & \underline{$1.096$} & $3.5$ & \underline{$28.2$} & $6.1$ \\
$3 \cdot 10^{-5}$ & \underline{$0.022$} & $1.099$ & \underline{$3.6$} & $25.8$ & $6.5$ \\
$1 \cdot 10^{-5}$ & $0.038$ & $1.103$ & $3.4$ & $25.7$ & $6.7$ \\
\bottomrule
\end{tabular}
\end{table}

%% file: tabs/sft_lr_rank_tables.tex
\section{Detailed SFT results}

\subsection{Llama-3.2-1B, Tulu-3}

\begin{center}
\textit{Llama-3.2-1B.}\\[2pt]
\begin{tabular}{llcccccc}
\toprule
\textbf{Method} & $r$ & $10^{-5}$ & $3{\cdot}10^{-5}$ & $10^{-4}$ & $3{\cdot}10^{-4}$ & $10^{-3}$ & $3{\cdot}10^{-3}$ \\
\midrule
LoRA & 1 & \cellcolor[HTML]{8BCF89}$0.822$ & \cellcolor[HTML]{6BC072}$0.811$ & \cellcolor[HTML]{66BD6F}$0.810$ & \cellcolor[HTML]{56B567}$\underline{0.805}$ & \cellcolor[HTML]{72C375}$0.813$ & \cellcolor[HTML]{BBE4B4}$0.839$ \\
 & 2 & \cellcolor[HTML]{86CC85}$0.820$ & \cellcolor[HTML]{5BB86A}$0.807$ & \cellcolor[HTML]{50B264}$0.803$ & \cellcolor[HTML]{43AC5E}$\underline{0.800}$ & \cellcolor[HTML]{76C578}$0.815$ & \cellcolor[HTML]{F0F9ED}$6.743$ \\
 & 4 & \cellcolor[HTML]{83CB82}$0.819$ & \cellcolor[HTML]{53B466}$0.804$ & \cellcolor[HTML]{3DA65A}$0.797$ & \cellcolor[HTML]{3DA65A}$\underline{0.797}$ & \cellcolor[HTML]{9CD797}$0.828$ & \cellcolor[HTML]{F0F9ED}$7.464$ \\
 & 8 & \cellcolor[HTML]{81CA81}$0.819$ & \cellcolor[HTML]{4DB163}$0.803$ & \cellcolor[HTML]{349D53}$\underline{0.792}$ & \cellcolor[HTML]{39A257}$0.795$ & --- & --- \\
 & 16 & \cellcolor[HTML]{81CA81}$0.818$ & \cellcolor[HTML]{4AAF61}$0.802$ & \cellcolor[HTML]{2F974E}$\underline{0.790}$ & \cellcolor[HTML]{40AA5D}$0.799$ & --- & --- \\
\midrule
LoRA\textsuperscript{P} & 1 & \cellcolor[HTML]{B6E2AF}$0.838$ & \cellcolor[HTML]{9CD797}$0.827$ & \cellcolor[HTML]{90D18D}$\underline{0.823}$ & \cellcolor[HTML]{99D595}$0.827$ & \cellcolor[HTML]{F0F9ED}$2.997$ & --- \\
 & 4 & --- & --- & \cellcolor[HTML]{81CA81}$\underline{0.819}$ & \cellcolor[HTML]{83CB82}$0.819$ & --- & --- \\
\bottomrule
\end{tabular}
\end{center}

\subsection{Llama-3.2-1B, OpenThoughts}

\begin{center}
\textit{Llama-3.2-1B.}\\[2pt]
\begin{tabular}{llcccccc}
\toprule
\textbf{Method} & $r$ & $10^{-5}$ & $3{\cdot}10^{-5}$ & $10^{-4}$ & $3{\cdot}10^{-4}$ & $10^{-3}$ & $3{\cdot}10^{-3}$ \\
\midrule
LoRA & 1 & \cellcolor[HTML]{F0F9ED}$1.483$ & \cellcolor[HTML]{F0F9ED}$1.436$ & \cellcolor[HTML]{F0F9ED}$1.410$ & \cellcolor[HTML]{F0F9ED}$1.397$ & \cellcolor[HTML]{F0F9ED}$\underline{1.393}$ & \cellcolor[HTML]{F0F9ED}$1.401$ \\
 & 4 & --- & --- & --- & \cellcolor[HTML]{C7E9C0}$\underline{1.336}$ & \cellcolor[HTML]{D6EFD0}$1.348$ & --- \\
 & 16 & --- & --- & --- & \cellcolor[HTML]{6BC072}$\underline{1.285}$ & \cellcolor[HTML]{7DC87E}$1.294$ & --- \\
 & 64 & --- & --- & --- & \cellcolor[HTML]{309950}$1.251$ & \cellcolor[HTML]{2F974E}$\underline{1.250}$ & --- \\
\midrule
LoRA\textsuperscript{P} & 1 & \cellcolor[HTML]{F0F9ED}$1.544$ & \cellcolor[HTML]{F0F9ED}$1.514$ & \cellcolor[HTML]{F0F9ED}$1.492$ & \cellcolor[HTML]{F0F9ED}$1.484$ & \cellcolor[HTML]{F0F9ED}$\underline{1.482}$ & \cellcolor[HTML]{F0F9ED}$5.248$ \\
 & 4 & --- & --- & --- & \cellcolor[HTML]{F0F9ED}$1.461$ & \cellcolor[HTML]{F0F9ED}$\underline{1.453}$ & \cellcolor[HTML]{F0F9ED}$4.909$ \\
 & 16 & --- & --- & --- & \cellcolor[HTML]{F0F9ED}$1.451$ & \cellcolor[HTML]{F0F9ED}$\underline{1.434}$ & \cellcolor[HTML]{F0F9ED}$5.090$ \\
\bottomrule
\end{tabular}
\end{center}

\subsection{Llama-3.2-1B, GSM8K}

\begin{center}
\textit{Llama-3.2-1B.}\\[2pt]
\begin{tabular}{llcccccc}
\toprule
\textbf{Method} & $r$ & $10^{-5}$ & $3{\cdot}10^{-5}$ & $10^{-4}$ & $3{\cdot}10^{-4}$ & $10^{-3}$ & $3{\cdot}10^{-3}$ \\
\midrule
LoRA\textsuperscript{P} & 1 & \cellcolor[HTML]{86CC85}$0.588$ & \cellcolor[HTML]{45AD5F}$0.574$ & \cellcolor[HTML]{329B51}$0.568$ & \cellcolor[HTML]{2F974E}$\underline{0.566}$ & \cellcolor[HTML]{C0E6B9}$0.603$ & \cellcolor[HTML]{F0F9ED}$4.124$ \\
\bottomrule
\end{tabular}
\end{center}

\subsection{Llama-3.2-1B-Base, Tulu-3}

\begin{center}
\textit{Llama-3.2-1B-Base.}\\[2pt]
\begin{tabular}{llcccccc}
\toprule
\textbf{Method} & $r$ & $10^{-5}$ & $3{\cdot}10^{-5}$ & $10^{-4}$ & $3{\cdot}10^{-4}$ & $10^{-3}$ & $3{\cdot}10^{-3}$ \\
\midrule
LoRA & 1 & \cellcolor[HTML]{C7E9C0}$0.919$ & \cellcolor[HTML]{A3DA9D}$0.903$ & \cellcolor[HTML]{95D391}$\underline{0.898}$ & \cellcolor[HTML]{B1E0AB}$0.910$ & \cellcolor[HTML]{F0F9ED}$0.979$ & --- \\
 & 4 & --- & --- & \cellcolor[HTML]{55B567}$\underline{0.876}$ & --- & --- & --- \\
 & 16 & --- & --- & \cellcolor[HTML]{37A055}$\underline{0.864}$ & --- & --- & --- \\
 & 64 & --- & --- & \cellcolor[HTML]{2F974E}$\underline{0.860}$ & --- & --- & --- \\
\midrule
LoRA\textsuperscript{P} & 1 & \cellcolor[HTML]{F0F9ED}$0.952$ & \cellcolor[HTML]{E6F5E1}$0.936$ & \cellcolor[HTML]{D5EFCF}$\underline{0.927}$ & \cellcolor[HTML]{DBF1D6}$0.930$ & \cellcolor[HTML]{F0F9ED}$6.118$ & --- \\
 & 4 & --- & --- & \cellcolor[HTML]{C6E8BF}$\underline{0.918}$ & --- & --- & --- \\
 & 16 & --- & --- & \cellcolor[HTML]{C0E6B9}$\underline{0.916}$ & --- & --- & --- \\
\bottomrule
\end{tabular}
\end{center}

\subsection{Llama-3.2-1B-Base, OpenThoughts}

\begin{center}
\textit{Llama-3.2-1B-Base.}\\[2pt]
\begin{tabular}{llcccccc}
\toprule
\textbf{Method} & $r$ & $10^{-5}$ & $3{\cdot}10^{-5}$ & $10^{-4}$ & $3{\cdot}10^{-4}$ & $10^{-3}$ & $3{\cdot}10^{-3}$ \\
\midrule
LoRA & 1 & --- & --- & --- & \cellcolor[HTML]{F0F9ED}$\underline{1.444}$ & \cellcolor[HTML]{F0F9ED}$8.064$ & --- \\
 & 4 & --- & --- & --- & \cellcolor[HTML]{DBF1D5}$\underline{1.375}$ & --- & --- \\
 & 16 & --- & --- & --- & \cellcolor[HTML]{76C578}$\underline{1.311}$ & --- & --- \\
 & 64 & --- & --- & --- & \cellcolor[HTML]{2F974E}$\underline{1.271}$ & --- & --- \\
\midrule
LoRA\textsuperscript{P} & 1 & --- & --- & --- & \cellcolor[HTML]{F0F9ED}$\underline{1.552}$ & \cellcolor[HTML]{F0F9ED}$1.553$ & --- \\
 & 4 & --- & --- & --- & \cellcolor[HTML]{F0F9ED}$\underline{1.529}$ & --- & --- \\
 & 16 & --- & --- & --- & \cellcolor[HTML]{F0F9ED}$\underline{1.520}$ & --- & --- \\
 & 64 & --- & --- & --- & \cellcolor[HTML]{F0F9ED}$\underline{1.518}$ & --- & --- \\
\midrule
DiReFT\textsuperscript{A} & 1 & --- & --- & --- & \cellcolor[HTML]{F0F9ED}$1.554$ & \cellcolor[HTML]{F0F9ED}$\underline{1.546}$ & --- \\
 & 4 & --- & --- & --- & --- & \cellcolor[HTML]{F0F9ED}$\underline{1.485}$ & --- \\
 & 16 & --- & --- & --- & --- & \cellcolor[HTML]{F0F9ED}$\underline{1.423}$ & --- \\
 & 64 & --- & --- & --- & --- & \cellcolor[HTML]{D2EDCC}$\underline{1.368}$ & --- \\
\midrule
DiReFT\textsuperscript{P} & 1 & --- & --- & \cellcolor[HTML]{F0F9ED}$1.640$ & \cellcolor[HTML]{F0F9ED}$1.616$ & \cellcolor[HTML]{F0F9ED}$\underline{1.605}$ & --- \\
 & 4 & --- & --- & --- & --- & \cellcolor[HTML]{F0F9ED}$\underline{1.577}$ & --- \\
 & 16 & --- & --- & --- & --- & \cellcolor[HTML]{F0F9ED}$\underline{1.562}$ & --- \\
 & 64 & --- & --- & --- & --- & \cellcolor[HTML]{F0F9ED}$\underline{1.552}$ & --- \\
\bottomrule
\end{tabular}
\end{center}

\subsection{Llama-3.1-8B, Tulu-3}

\begin{center}
\textit{Llama-3.1-8B.}\\[2pt]
\begin{tabular}{llcccccc}
\toprule
\textbf{Method} & $r$ & $10^{-5}$ & $3{\cdot}10^{-5}$ & $10^{-4}$ & $3{\cdot}10^{-4}$ & $10^{-3}$ & $3{\cdot}10^{-3}$ \\
\midrule
LoRA & 1 & \cellcolor[HTML]{79C67A}$0.632$ & \cellcolor[HTML]{56B567}$\underline{0.624}$ & \cellcolor[HTML]{58B668}$0.624$ & \cellcolor[HTML]{ACDEA6}$0.646$ & \cellcolor[HTML]{F0F9ED}$7.473$ & \cellcolor[HTML]{F0F9ED}$6.916$ \\
 & 2 & \cellcolor[HTML]{73C476}$0.631$ & \cellcolor[HTML]{50B264}$0.623$ & \cellcolor[HTML]{48AE60}$\underline{0.621}$ & \cellcolor[HTML]{A2D99C}$0.643$ & \cellcolor[HTML]{F0F9ED}$7.455$ & \cellcolor[HTML]{F0F9ED}$7.486$ \\
 & 4 & \cellcolor[HTML]{72C375}$0.630$ & \cellcolor[HTML]{4AAF61}$0.621$ & \cellcolor[HTML]{3CA559}$\underline{0.617}$ & \cellcolor[HTML]{A3DA9D}$0.643$ & \cellcolor[HTML]{F0F9ED}$6.047$ & \cellcolor[HTML]{F0F9ED}$6.281$ \\
 & 8 & \cellcolor[HTML]{72C375}$0.630$ & \cellcolor[HTML]{48AE60}$0.621$ & \cellcolor[HTML]{369F54}$\underline{0.615}$ & \cellcolor[HTML]{8DD08A}$0.637$ & \cellcolor[HTML]{F0F9ED}$7.486$ & \cellcolor[HTML]{F0F9ED}$7.015$ \\
 & 16 & \cellcolor[HTML]{72C375}$0.630$ & \cellcolor[HTML]{46AE60}$0.620$ & \cellcolor[HTML]{319A50}$\underline{0.613}$ & \cellcolor[HTML]{7AC77B}$0.633$ & \cellcolor[HTML]{F0F9ED}$6.733$ & \cellcolor[HTML]{F0F9ED}$6.327$ \\
 & 32 & \cellcolor[HTML]{72C375}$0.630$ & \cellcolor[HTML]{45AD5F}$0.620$ & \cellcolor[HTML]{2F974E}$\underline{0.612}$ & \cellcolor[HTML]{6ABF71}$0.628$ & \cellcolor[HTML]{F0F9ED}$4.376$ & \cellcolor[HTML]{F0F9ED}$6.539$ \\
\midrule
LoRA\textsuperscript{P} & 1 & \cellcolor[HTML]{A7DBA0}$0.644$ & \cellcolor[HTML]{87CD86}$0.636$ & \cellcolor[HTML]{7CC87C}$\underline{0.633}$ & \cellcolor[HTML]{A0D99B}$0.642$ & \cellcolor[HTML]{F0F9ED}$7.208$ & --- \\
\bottomrule
\end{tabular}
\end{center}

\subsection{Llama-3.1-8B, GSM8K}

\begin{center}
\textit{Llama-3.1-8B.}\\[2pt]
\begin{tabular}{llcccccc}
\toprule
\textbf{Method} & $r$ & $10^{-5}$ & $3{\cdot}10^{-5}$ & $10^{-4}$ & $3{\cdot}10^{-4}$ & $10^{-3}$ & $3{\cdot}10^{-3}$ \\
\midrule
LoRA & 4 & \cellcolor[HTML]{F0F9ED}$0.535$ & \cellcolor[HTML]{CEECC8}$0.499$ & \cellcolor[HTML]{4DB163}$0.473$ & \cellcolor[HTML]{3FA85B}$\underline{0.470}$ & \cellcolor[HTML]{F0F9ED}$0.570$ & --- \\
 & 8 & \cellcolor[HTML]{F0F9ED}$0.536$ & \cellcolor[HTML]{D2EDCC}$0.501$ & \cellcolor[HTML]{4BB062}$0.472$ & \cellcolor[HTML]{369F54}$\underline{0.467}$ & \cellcolor[HTML]{F0F9ED}$0.560$ & --- \\
 & 16 & \cellcolor[HTML]{F0F9ED}$0.540$ & \cellcolor[HTML]{DBF1D5}$0.503$ & \cellcolor[HTML]{58B668}$0.474$ & \cellcolor[HTML]{319A50}$\underline{0.466}$ & \cellcolor[HTML]{F0F9ED}$0.549$ & --- \\
 & 32 & \cellcolor[HTML]{F0F9ED}$0.541$ & \cellcolor[HTML]{DBF1D6}$0.503$ & \cellcolor[HTML]{55B567}$0.474$ & \cellcolor[HTML]{349D53}$\underline{0.467}$ & \cellcolor[HTML]{F0F9ED}$0.543$ & --- \\
 & 64 & \cellcolor[HTML]{F0F9ED}$0.541$ & \cellcolor[HTML]{DBF1D5}$0.503$ & \cellcolor[HTML]{55B567}$0.474$ & \cellcolor[HTML]{2F974E}$\underline{0.465}$ & \cellcolor[HTML]{F0F9ED}$0.542$ & --- \\
\bottomrule
\end{tabular}
\end{center}

\subsection{Llama-3.1-8B-Base, Tulu-3}

\begin{center}
\textit{Llama-3.1-8B-Base.}\\[2pt]
\begin{tabular}{llcccccc}
\toprule
\textbf{Method} & $r$ & $10^{-5}$ & $3{\cdot}10^{-5}$ & $10^{-4}$ & $3{\cdot}10^{-4}$ & $10^{-3}$ & $3{\cdot}10^{-3}$ \\
\midrule
LoRA & 1 & --- & \cellcolor[HTML]{62BB6D}$0.651$ & \cellcolor[HTML]{60BA6C}$\underline{0.650}$ & \cellcolor[HTML]{BEE5B8}$0.677$ & --- & --- \\
 & 4 & --- & --- & \cellcolor[HTML]{3BA458}$\underline{0.641}$ & --- & --- & --- \\
 & 16 & --- & --- & \cellcolor[HTML]{2F974E}$\underline{0.636}$ & --- & --- & --- \\
\midrule
LoRA\textsuperscript{P} & 1 & --- & \cellcolor[HTML]{62BB6D}$0.651$ & \cellcolor[HTML]{5EB96B}$\underline{0.650}$ & \cellcolor[HTML]{BEE5B8}$0.677$ & --- & --- \\
 & 4 & --- & --- & \cellcolor[HTML]{3CA559}$\underline{0.641}$ & --- & --- & --- \\
\midrule
DiReFT\textsuperscript{A} & 1 & --- & \cellcolor[HTML]{F0F9ED}$0.725$ & \cellcolor[HTML]{DEF2D9}$0.689$ & \cellcolor[HTML]{B1E0AB}$0.673$ & \cellcolor[HTML]{B0DFAA}$\underline{0.672}$ & --- \\
 & 4 & --- & --- & --- & --- & \cellcolor[HTML]{9BD696}$\underline{0.666}$ & --- \\
 & 16 & --- & --- & --- & --- & \cellcolor[HTML]{8BCF89}$\underline{0.662}$ & --- \\
 & 64 & --- & --- & --- & --- & \cellcolor[HTML]{8ED08B}$\underline{0.663}$ & --- \\
\midrule
DiReFT\textsuperscript{P} & 1 & --- & \cellcolor[HTML]{F0F9ED}$0.732$ & \cellcolor[HTML]{F0F9ED}$0.702$ & \cellcolor[HTML]{D9F0D3}$0.687$ & \cellcolor[HTML]{C9EAC2}$\underline{0.680}$ & --- \\
 & 4 & --- & --- & --- & --- & \cellcolor[HTML]{B0DFAA}$\underline{0.673}$ & --- \\
 & 16 & --- & --- & --- & --- & \cellcolor[HTML]{AEDEA7}$\underline{0.672}$ & --- \\
 & 64 & --- & --- & --- & --- & \cellcolor[HTML]{ABDDA5}$\underline{0.671}$ & --- \\
\bottomrule
\end{tabular}
\end{center}

\section{SFT (downstream performance)}

\subsection{Llama-3.2-1B, Tulu-3}

\begin{center}
\textit{All methods on Tulu-3, Llama-3.2-1B.}\\[2pt]
\begin{adjustbox}{max width=\textwidth}
\begin{tabular}{llcccccccccccccccccc}
\toprule
\textbf{Method} & $r$ & \multicolumn{3}{c}{$10^{-5}$} & \multicolumn{3}{c}{$3{\cdot}10^{-5}$} & \multicolumn{3}{c}{$10^{-4}$} & \multicolumn{3}{c}{$3{\cdot}10^{-4}$} & \multicolumn{3}{c}{$10^{-3}$} & \multicolumn{3}{c}{$3{\cdot}10^{-3}$} \\
\cmidrule(lr){3-5} \cmidrule(lr){6-8} \cmidrule(lr){9-11} \cmidrule(lr){12-14} \cmidrule(lr){15-17} \cmidrule(lr){18-20}
 & & IF & ML & GS & IF & ML & GS & IF & ML & GS & IF & ML & GS & IF & ML & GS & IF & ML & GS \\
\midrule
LoRA & 1 & --- & --- & --- & --- & --- & --- & \cellcolor[HTML]{2F974E}$\underline{0.414}$ & \cellcolor[HTML]{40AA5D}$\underline{0.297}$ & \cellcolor[HTML]{B0DFAA}$\underline{0.303}$ & --- & --- & --- & --- & --- & --- & --- & --- & --- \\
 & 2 & --- & --- & --- & --- & --- & --- & --- & --- & --- & \cellcolor[HTML]{80CA80}$\underline{0.397}$ & \cellcolor[HTML]{C7E9C0}$\underline{0.259}$ & \cellcolor[HTML]{A7DBA0}$\underline{0.312}$ & --- & --- & --- & --- & --- & --- \\
 & 4 & --- & --- & --- & --- & --- & --- & --- & --- & --- & \cellcolor[HTML]{3BA458}$\underline{0.410}$ & \cellcolor[HTML]{91D28E}$\underline{0.276}$ & \cellcolor[HTML]{87CD86}$\underline{0.334}$ & --- & --- & --- & --- & --- & --- \\
 & 8 & --- & --- & --- & --- & --- & --- & \cellcolor[HTML]{B5E1AE}$\underline{0.386}$ & \cellcolor[HTML]{2F974E}$\underline{0.305}$ & \cellcolor[HTML]{37A055}$\underline{0.387}$ & --- & --- & --- & --- & --- & --- & --- & --- & --- \\
 & 16 & --- & --- & --- & --- & --- & --- & \cellcolor[HTML]{C4E8BD}$\underline{0.383}$ & \cellcolor[HTML]{70C274}$\underline{0.285}$ & \cellcolor[HTML]{84CC83}$\underline{0.335}$ & --- & --- & --- & --- & --- & --- & --- & --- & --- \\
\midrule
LoRA\textsuperscript{P} & 1 & --- & --- & --- & --- & --- & --- & \cellcolor[HTML]{D2EDCC}$\underline{0.379}$ & \cellcolor[HTML]{C3E7BC}$\underline{0.260}$ & \cellcolor[HTML]{2F974E}$\underline{0.396}$ & --- & --- & --- & --- & --- & --- & --- & --- & --- \\
\bottomrule
\end{tabular}
\end{adjustbox}
\end{center}

\subsection{Llama-3.2-1B, OpenThoughts}

\begin{center}
\textit{All methods on OpenThoughts, Llama-3.2-1B.}\\[2pt]
\begin{adjustbox}{max width=\textwidth}
\begin{tabular}{llcccccccccccccccccc}
\toprule
\textbf{Method} & $r$ & \multicolumn{3}{c}{$10^{-5}$} & \multicolumn{3}{c}{$3{\cdot}10^{-5}$} & \multicolumn{3}{c}{$10^{-4}$} & \multicolumn{3}{c}{$3{\cdot}10^{-4}$} & \multicolumn{3}{c}{$10^{-3}$} & \multicolumn{3}{c}{$3{\cdot}10^{-3}$} \\
\cmidrule(lr){3-5} \cmidrule(lr){6-8} \cmidrule(lr){9-11} \cmidrule(lr){12-14} \cmidrule(lr){15-17} \cmidrule(lr){18-20}
 & & GS & MA & MP & GS & MA & MP & GS & MA & MP & GS & MA & MP & GS & MA & MP & GS & MA & MP \\
\midrule
LoRA & 1 & \cellcolor[HTML]{9CD797}$\underline{0.192}$ & \cellcolor[HTML]{42AB5D}$\underline{0.203}$ & \cellcolor[HTML]{2F974E}$\underline{0.143}$ & \cellcolor[HTML]{A9DCA3}$0.171$ & \cellcolor[HTML]{7CC87C}$0.158$ & \cellcolor[HTML]{3CA559}$0.132$ & \cellcolor[HTML]{B0DFAA}$0.158$ & \cellcolor[HTML]{73C476}$0.166$ & \cellcolor[HTML]{4BB062}$0.121$ & \cellcolor[HTML]{CAEAC3}$0.113$ & \cellcolor[HTML]{81CA81}$0.152$ & \cellcolor[HTML]{60BA6C}$0.110$ & \cellcolor[HTML]{D7EFD1}$0.083$ & \cellcolor[HTML]{8ACE88}$0.145$ & \cellcolor[HTML]{66BD6F}$0.107$ & \cellcolor[HTML]{C2E7BB}$0.127$ & \cellcolor[HTML]{73C476}$0.165$ & \cellcolor[HTML]{7CC87C}$0.094$ \\
\midrule
LoRA\textsuperscript{P} & 1 & \cellcolor[HTML]{8ACE88}$0.218$ & \cellcolor[HTML]{3DA65A}$0.208$ & \cellcolor[HTML]{4EB264}$\underline{0.119}$ & \cellcolor[HTML]{5EB96B}$\underline{0.278}$ & \cellcolor[HTML]{2F974E}$\underline{0.227}$ & \cellcolor[HTML]{8DD08A}$0.085$ & \cellcolor[HTML]{6DC072}$0.260$ & \cellcolor[HTML]{37A055}$0.216$ & \cellcolor[HTML]{BBE4B4}$0.053$ & \cellcolor[HTML]{BAE3B3}$0.142$ & \cellcolor[HTML]{DFF3DA}$0.055$ & \cellcolor[HTML]{B6E2AF}$0.057$ & \cellcolor[HTML]{CEECC8}$0.101$ & \cellcolor[HTML]{C2E7BB}$0.092$ & \cellcolor[HTML]{A8DCA2}$0.067$ & \cellcolor[HTML]{EFF9EB}$0.017$ & \cellcolor[HTML]{EEF8EA}$0.029$ & \cellcolor[HTML]{F0F9ED}$0.000$ \\
\bottomrule
\end{tabular}
\end{adjustbox}
\end{center}

\subsection{Llama-3.2-1B-Base, Tulu-3}

\begin{center}
\textit{All methods on Tulu-3, Llama-3.2-1B-Base.}\\[2pt]
\begin{adjustbox}{max width=\textwidth}
\begin{tabular}{llcccccccccccccccccc}
\toprule
\textbf{Method} & $r$ & \multicolumn{3}{c}{$10^{-5}$} & \multicolumn{3}{c}{$3{\cdot}10^{-5}$} & \multicolumn{3}{c}{$10^{-4}$} & \multicolumn{3}{c}{$3{\cdot}10^{-4}$} & \multicolumn{3}{c}{$10^{-3}$} & \multicolumn{3}{c}{$3{\cdot}10^{-3}$} \\
\cmidrule(lr){3-5} \cmidrule(lr){6-8} \cmidrule(lr){9-11} \cmidrule(lr){12-14} \cmidrule(lr){15-17} \cmidrule(lr){18-20}
 & & IF & ML & GS & IF & ML & GS & IF & ML & GS & IF & ML & GS & IF & ML & GS & IF & ML & GS \\
\midrule
LoRA & 1 & \cellcolor[HTML]{BAE3B3}$0.161$ & \cellcolor[HTML]{53B466}$\underline{0.253}$ & \cellcolor[HTML]{B4E1AD}$0.063$ & \cellcolor[HTML]{81CA81}$0.196$ & \cellcolor[HTML]{66BD6F}$0.244$ & \cellcolor[HTML]{94D390}$0.077$ & \cellcolor[HTML]{63BC6E}$\underline{0.213}$ & \cellcolor[HTML]{70C274}$0.240$ & \cellcolor[HTML]{8DD08A}$\underline{0.079}$ & \cellcolor[HTML]{6EC173}$0.207$ & \cellcolor[HTML]{73C476}$0.238$ & \cellcolor[HTML]{99D595}$0.074$ & \cellcolor[HTML]{91D28E}$0.187$ & \cellcolor[HTML]{58B668}$0.251$ & \cellcolor[HTML]{CBEAC4}$0.052$ & --- & --- & --- \\
 & 4 & --- & --- & --- & --- & --- & --- & \cellcolor[HTML]{60BA6C}$\underline{0.214}$ & \cellcolor[HTML]{369F54}$\underline{0.271}$ & \cellcolor[HTML]{7FC97F}$\underline{0.084}$ & \cellcolor[HTML]{BCE4B5}$0.159$ & \cellcolor[HTML]{C3E7BC}$0.191$ & \cellcolor[HTML]{F0F9ED}$0.026$ & --- & --- & --- & --- & --- & --- \\
 & 16 & --- & --- & --- & --- & --- & --- & \cellcolor[HTML]{75C477}$\underline{0.203}$ & \cellcolor[HTML]{55B567}$\underline{0.252}$ & \cellcolor[HTML]{73C476}$\underline{0.089}$ & \cellcolor[HTML]{CBEAC4}$0.148$ & \cellcolor[HTML]{CDECC7}$0.183$ & \cellcolor[HTML]{DFF3DA}$0.040$ & --- & --- & --- & --- & --- & --- \\
\bottomrule
\end{tabular}
\end{adjustbox}
\end{center}

\subsection{Llama-3.1-8B, Tulu-3}

\begin{center}
\textit{All methods on Tulu-3, Llama-3.1-8B.}\\[2pt]
\begin{adjustbox}{max width=\textwidth}
\begin{tabular}{llcccccccccccccccccc}
\toprule
\textbf{Method} & $r$ & \multicolumn{3}{c}{$10^{-5}$} & \multicolumn{3}{c}{$3{\cdot}10^{-5}$} & \multicolumn{3}{c}{$10^{-4}$} & \multicolumn{3}{c}{$3{\cdot}10^{-4}$} & \multicolumn{3}{c}{$10^{-3}$} & \multicolumn{3}{c}{$3{\cdot}10^{-3}$} \\
\cmidrule(lr){3-5} \cmidrule(lr){6-8} \cmidrule(lr){9-11} \cmidrule(lr){12-14} \cmidrule(lr){15-17} \cmidrule(lr){18-20}
 & & IF & ML & GS & IF & ML & GS & IF & ML & GS & IF & ML & GS & IF & ML & GS & IF & ML & GS \\
\midrule
LoRA & 1 & --- & --- & --- & \cellcolor[HTML]{BBE4B4}$\underline{0.641}$ & \cellcolor[HTML]{9CD797}$\underline{0.567}$ & \cellcolor[HTML]{62BB6D}$\underline{0.703}$ & --- & --- & --- & --- & --- & --- & --- & --- & --- & --- & --- & --- \\
 & 4 & --- & --- & --- & --- & --- & --- & \cellcolor[HTML]{C6E8BF}$\underline{0.638}$ & \cellcolor[HTML]{A9DCA3}$\underline{0.559}$ & \cellcolor[HTML]{C9EAC2}$\underline{0.600}$ & --- & --- & --- & --- & --- & --- & --- & --- & --- \\
 & 8 & --- & --- & --- & --- & --- & --- & \cellcolor[HTML]{7CC87C}$\underline{0.662}$ & \cellcolor[HTML]{56B567}$\underline{0.600}$ & \cellcolor[HTML]{CAEAC3}$\underline{0.598}$ & --- & --- & --- & --- & --- & --- & --- & --- & --- \\
 & 16 & --- & --- & --- & --- & --- & --- & \cellcolor[HTML]{68BE70}$\underline{0.667}$ & \cellcolor[HTML]{58B668}$\underline{0.600}$ & \cellcolor[HTML]{F0F9ED}$\underline{0.534}$ & --- & --- & --- & --- & --- & --- & --- & --- & --- \\
 & 32 & --- & --- & --- & --- & --- & --- & \cellcolor[HTML]{9BD696}$\underline{0.652}$ & \cellcolor[HTML]{65BD6F}$\underline{0.594}$ & \cellcolor[HTML]{CDECC7}$\underline{0.593}$ & --- & --- & --- & --- & --- & --- & --- & --- & --- \\
\midrule
LoRA\textsuperscript{P} & 1 & --- & --- & --- & --- & --- & --- & \cellcolor[HTML]{70C274}$\underline{0.665}$ & \cellcolor[HTML]{81CA81}$\underline{0.580}$ & \cellcolor[HTML]{45AD5F}$\underline{0.727}$ & --- & --- & --- & --- & --- & --- & --- & --- & --- \\
\bottomrule
\end{tabular}
\end{adjustbox}
\end{center}

\subsection{Llama-3.1-8B-Base, Tulu-3}

\begin{center}
\textit{All methods on Tulu-3, Llama-3.1-8B-Base.}\\[2pt]
\begin{adjustbox}{max width=\textwidth}
\begin{tabular}{llcccccccccccccccccc}
\toprule
\textbf{Method} & $r$ & \multicolumn{3}{c}{$10^{-5}$} & \multicolumn{3}{c}{$3{\cdot}10^{-5}$} & \multicolumn{3}{c}{$10^{-4}$} & \multicolumn{3}{c}{$3{\cdot}10^{-4}$} & \multicolumn{3}{c}{$10^{-3}$} & \multicolumn{3}{c}{$3{\cdot}10^{-3}$} \\
\cmidrule(lr){3-5} \cmidrule(lr){6-8} \cmidrule(lr){9-11} \cmidrule(lr){12-14} \cmidrule(lr){15-17} \cmidrule(lr){18-20}
 & & IF & ML & GS & IF & ML & GS & IF & ML & GS & IF & ML & GS & IF & ML & GS & IF & ML & GS \\
\midrule
LoRA & 1 & --- & --- & --- & \cellcolor[HTML]{50B264}$0.455$ & \cellcolor[HTML]{2F984F}$0.545$ & \cellcolor[HTML]{4EB264}$\underline{0.490}$ & \cellcolor[HTML]{3EA75A}$\underline{0.492}$ & \cellcolor[HTML]{2F974E}$\underline{0.547}$ & \cellcolor[HTML]{63BC6E}$0.448$ & \cellcolor[HTML]{53B466}$0.449$ & \cellcolor[HTML]{309950}$0.541$ & \cellcolor[HTML]{63BC6E}$0.447$ & --- & --- & --- & --- & --- & --- \\
 & 4 & --- & --- & --- & --- & --- & --- & \cellcolor[HTML]{3CA559}$\underline{0.497}$ & \cellcolor[HTML]{3BA458}$\underline{0.505}$ & \cellcolor[HTML]{5BB86A}$\underline{0.464}$ & --- & --- & --- & --- & --- & --- & --- & --- & --- \\
 & 16 & --- & --- & --- & --- & --- & --- & \cellcolor[HTML]{2F974E}$\underline{0.534}$ & \cellcolor[HTML]{38A156}$\underline{0.514}$ & \cellcolor[HTML]{4EB264}$\underline{0.490}$ & --- & --- & --- & --- & --- & --- & --- & --- & --- \\
\midrule
LoRA\textsuperscript{P} & 1 & --- & --- & --- & \cellcolor[HTML]{72C375}$0.397$ & \cellcolor[HTML]{339C52}$\underline{0.531}$ & \cellcolor[HTML]{75C477}$0.409$ & \cellcolor[HTML]{55B567}$\underline{0.446}$ & \cellcolor[HTML]{359E53}$0.524$ & \cellcolor[HTML]{6EC173}$0.424$ & \cellcolor[HTML]{66BD6F}$0.416$ & \cellcolor[HTML]{60BA6C}$0.420$ & \cellcolor[HTML]{5EB96B}$\underline{0.457}$ & --- & --- & --- & --- & --- & --- \\
\bottomrule
\end{tabular}
\end{adjustbox}
\end{center}

\subsection{Llama-3.1-8B, GSM8K}

\begin{center}
\textit{All methods on GSM8K, Llama-3.1-8B (eval/reward\_full).}\\[2pt]
\begin{tabular}{llcccccc}
\toprule
\textbf{Method} & $r$ & $10^{-5}$ & $3{\cdot}10^{-5}$ & $10^{-4}$ & $3{\cdot}10^{-4}$ & $10^{-3}$ & $3{\cdot}10^{-3}$ \\
\midrule
DiReFT\textsuperscript{P} & 4 & \cellcolor[HTML]{F0F9ED}$0.184$ & \cellcolor[HTML]{73C476}$0.440$ & \cellcolor[HTML]{52B365}$0.486$ & \cellcolor[HTML]{46AE60}$\underline{0.501}$ & --- & --- \\
 & 8 & --- & \cellcolor[HTML]{6EC173}$0.447$ & \cellcolor[HTML]{52B365}$0.486$ & --- & \cellcolor[HTML]{329B51}$\underline{0.544}$ & --- \\
 & 16 & \cellcolor[HTML]{D3EECD}$0.270$ & --- & \cellcolor[HTML]{45AD5F}$0.504$ & \cellcolor[HTML]{4DB163}$0.493$ & \cellcolor[HTML]{2F974E}$\underline{0.552}$ & --- \\
 & 32 & \cellcolor[HTML]{CEECC8}$0.281$ & \cellcolor[HTML]{5DB96B}$0.472$ & \cellcolor[HTML]{5AB769}$0.475$ & \cellcolor[HTML]{46AE60}$0.501$ & \cellcolor[HTML]{2F984F}$\underline{0.550}$ & --- \\
 & 64 & \cellcolor[HTML]{C6E8BF}$0.300$ & \cellcolor[HTML]{58B668}$0.477$ & \cellcolor[HTML]{55B567}$0.481$ & \cellcolor[HTML]{40AA5D}$0.509$ & \cellcolor[HTML]{2F974E}$\underline{0.552}$ & --- \\
\midrule
LoReFT\textsuperscript{P} & 4 & \cellcolor[HTML]{8BCF89}$0.403$ & \cellcolor[HTML]{83CB82}$0.415$ & \cellcolor[HTML]{6BC072}$0.452$ & \cellcolor[HTML]{48AE60}$\underline{0.500}$ & \cellcolor[HTML]{55B567}$0.481$ & --- \\
 & 8 & \cellcolor[HTML]{70C274}$0.444$ & \cellcolor[HTML]{66BD6F}$0.458$ & \cellcolor[HTML]{53B466}$0.484$ & \cellcolor[HTML]{72C375}$0.441$ & \cellcolor[HTML]{40AA5D}$\underline{0.509}$ & --- \\
\bottomrule
\end{tabular}
\end{center}

\section{Detailed RL results}

\subsection{Qwen2.5-0.5B, GSM8K}

\begin{center}
\textit{Qwen2.5-0.5B.}\\[2pt]
\begin{tabular}{llccccccc}
\toprule
\textbf{Method} & $r$ & $10^{-6}$ & $3{\cdot}10^{-6}$ & $10^{-5}$ & $3{\cdot}10^{-5}$ & $10^{-4}$ & $3{\cdot}10^{-4}$ & $10^{-3}$ \\
\midrule
LoRA & 1 & \cellcolor[HTML]{50B264}$0.444$ & \cellcolor[HTML]{3DA65A}$0.488$ & \cellcolor[HTML]{359E53}$\underline{0.515}$ & \cellcolor[HTML]{37A055}$0.509$ & \cellcolor[HTML]{66BD6F}$0.400$ & \cellcolor[HTML]{F0F9ED}$0.000$ & \cellcolor[HTML]{F0F9ED}$0.000$ \\
 & 4 & --- & --- & \cellcolor[HTML]{319A50}$0.527$ & \cellcolor[HTML]{2F974E}$\underline{0.537}$ & --- & --- & --- \\
 & 8 & \cellcolor[HTML]{75C477}$0.371$ & \cellcolor[HTML]{4BB062}$0.453$ & \cellcolor[HTML]{359E53}$0.515$ & \cellcolor[HTML]{339C52}$\underline{0.520}$ & \cellcolor[HTML]{37A055}$0.509$ & \cellcolor[HTML]{EFF9EC}$0.005$ & \cellcolor[HTML]{F0F9ED}$0.000$ \\
 & 16 & --- & --- & \cellcolor[HTML]{38A156}$0.506$ & \cellcolor[HTML]{339C52}$\underline{0.520}$ & --- & --- & --- \\
 & 32 & \cellcolor[HTML]{79C67A}$0.360$ & \cellcolor[HTML]{50B264}$0.443$ & \cellcolor[HTML]{37A055}$\underline{0.509}$ & \cellcolor[HTML]{3BA458}$0.497$ & \cellcolor[HTML]{3DA65A}$0.490$ & \cellcolor[HTML]{EDF8E9}$0.020$ & \cellcolor[HTML]{F0F9ED}$0.000$ \\
 & 64 & --- & --- & \cellcolor[HTML]{9CD797}$0.281$ & \cellcolor[HTML]{37A055}$\underline{0.510}$ & --- & --- & --- \\
\midrule
LoRA\textsuperscript{P} & 1 & \cellcolor[HTML]{5EB96B}$0.416$ & \cellcolor[HTML]{95D391}$0.296$ & \cellcolor[HTML]{EFF9EB}$0.010$ & \cellcolor[HTML]{50B264}$\underline{0.444}$ & \cellcolor[HTML]{ECF8E8}$0.026$ & --- & --- \\
\bottomrule
\end{tabular}
\end{center}

\subsection{Qwen2.5-0.5B, MATH}

\begin{center}
\textit{Qwen2.5-0.5B.}\\[2pt]
\begin{tabular}{llccccccc}
\toprule
\textbf{Method} & $r$ & $10^{-6}$ & $3{\cdot}10^{-6}$ & $10^{-5}$ & $3{\cdot}10^{-5}$ & $10^{-4}$ & $3{\cdot}10^{-4}$ & $10^{-3}$ \\
\midrule
LoRA & 1 & --- & \cellcolor[HTML]{3EA75A}$0.302$ & \cellcolor[HTML]{45AD5F}$0.289$ & \cellcolor[HTML]{3AA357}$\underline{0.310}$ & \cellcolor[HTML]{E7F6E3}$0.031$ & --- & --- \\
 & 8 & --- & --- & --- & \cellcolor[HTML]{369F54}$\underline{0.318}$ & --- & --- & --- \\
 & 32 & --- & --- & --- & \cellcolor[HTML]{309950}$\underline{0.330}$ & --- & --- & --- \\
\midrule
LoRA\textsuperscript{P} & 1 & --- & \cellcolor[HTML]{4AAF61}$0.283$ & \cellcolor[HTML]{55B567}$0.270$ & \cellcolor[HTML]{359E53}$\underline{0.319}$ & \cellcolor[HTML]{EFF9EB}$0.005$ & --- & --- \\
 & 8 & --- & --- & --- & \cellcolor[HTML]{2F974E}$\underline{0.333}$ & --- & --- & --- \\
 & 32 & --- & --- & --- & \cellcolor[HTML]{2F984F}$\underline{0.332}$ & --- & --- & --- \\
\midrule
DiReFT\textsuperscript{A} & 1 & \cellcolor[HTML]{76C578}$0.227$ & \cellcolor[HTML]{75C477}$0.229$ & \cellcolor[HTML]{72C375}$0.234$ & \cellcolor[HTML]{78C679}$0.226$ & \cellcolor[HTML]{46AE60}$0.288$ & \cellcolor[HTML]{339C52}$\underline{0.324}$ & \cellcolor[HTML]{48AE60}$0.285$ \\
\bottomrule
\end{tabular}
\end{center}

\subsection{Qwen2.5-0.5B, HumanEval}

\begin{center}
\textit{Qwen2.5-0.5B.}\\[2pt]
\begin{tabular}{llccccccc}
\toprule
\textbf{Method} & $r$ & $10^{-6}$ & $3{\cdot}10^{-6}$ & $10^{-5}$ & $3{\cdot}10^{-5}$ & $10^{-4}$ & $3{\cdot}10^{-4}$ & $10^{-3}$ \\
\midrule
LoRA & 1 & --- & --- & --- & --- & \cellcolor[HTML]{42AB5D}$\underline{0.238}$ & --- & --- \\
 & 8 & --- & --- & --- & --- & \cellcolor[HTML]{E9F7E5}$\underline{0.110}$ & --- & --- \\
 & 32 & --- & --- & --- & --- & \cellcolor[HTML]{2F974E}$\underline{0.256}$ & --- & --- \\
\midrule
LoRA\textsuperscript{P} & 1 & --- & --- & --- & --- & \cellcolor[HTML]{90D18D}$\underline{0.189}$ & --- & --- \\
 & 8 & --- & --- & --- & --- & \cellcolor[HTML]{E9F7E5}$\underline{0.110}$ & --- & --- \\
 & 32 & --- & --- & --- & --- & \cellcolor[HTML]{60BA6C}$\underline{0.220}$ & --- & --- \\
\midrule
DiReFT\textsuperscript{A} & 1 & --- & --- & --- & --- & --- & \cellcolor[HTML]{F0F9ED}$\underline{0.098}$ & --- \\
 & 8 & --- & --- & --- & --- & \cellcolor[HTML]{E5F5E0}$\underline{0.116}$ & --- & --- \\
 & 32 & --- & --- & --- & --- & --- & \cellcolor[HTML]{C0E6B9}$\underline{0.152}$ & --- \\
\midrule
DiReFT\textsuperscript{P} & 1 & --- & --- & --- & --- & --- & \cellcolor[HTML]{E5F5E0}$\underline{0.116}$ & --- \\
 & 8 & --- & --- & --- & --- & --- & \cellcolor[HTML]{B8E3B2}$\underline{0.159}$ & --- \\
 & 32 & --- & --- & --- & --- & --- & \cellcolor[HTML]{ECF8E8}$\underline{0.104}$ & --- \\
\bottomrule
\end{tabular}
\end{center}

\subsection{Llama-3.1-8B, GSM8K}

\begin{center}
\textit{Llama-3.1-8B.}\\[2pt]
\begin{tabular}{llccccccc}
\toprule
\textbf{Method} & $r$ & $10^{-6}$ & $3{\cdot}10^{-6}$ & $10^{-5}$ & $3{\cdot}10^{-5}$ & $10^{-4}$ & $3{\cdot}10^{-4}$ & $10^{-3}$ \\
\midrule
LoRA & 1 & --- & --- & \cellcolor[HTML]{3DA65A}$0.670$ & \cellcolor[HTML]{339C52}$\underline{0.712}$ & --- & --- & --- \\
 & 4 & \cellcolor[HTML]{BCE4B5}$0.265$ & \cellcolor[HTML]{B6E2AF}$0.289$ & \cellcolor[HTML]{80CA80}$0.470$ & \cellcolor[HTML]{50B264}$\underline{0.607}$ & \cellcolor[HTML]{E8F6E3}$0.064$ & \cellcolor[HTML]{EFF9EB}$0.011$ & \cellcolor[HTML]{F0F9ED}$0.000$ \\
 & 8 & \cellcolor[HTML]{BCE4B5}$0.269$ & \cellcolor[HTML]{BBE4B4}$0.272$ & \cellcolor[HTML]{6ABF71}$0.534$ & \cellcolor[HTML]{2F974E}$\underline{0.733}$ & \cellcolor[HTML]{EDF8EA}$0.021$ & \cellcolor[HTML]{F0F9EC}$0.004$ & \cellcolor[HTML]{F0F9ED}$0.000$ \\
 & 16 & \cellcolor[HTML]{BBE4B4}$0.270$ & \cellcolor[HTML]{9ED798}$0.376$ & \cellcolor[HTML]{53B466}$\underline{0.600}$ & \cellcolor[HTML]{A2D99C}$0.362$ & \cellcolor[HTML]{EEF8EA}$0.017$ & \cellcolor[HTML]{DDF2D8}$0.123$ & \cellcolor[HTML]{F0F9ED}$0.000$ \\
 & 32 & \cellcolor[HTML]{B8E3B2}$0.281$ & \cellcolor[HTML]{92D28F}$0.411$ & \cellcolor[HTML]{42AB5D}$0.647$ & \cellcolor[HTML]{3FA95C}$\underline{0.653}$ & \cellcolor[HTML]{EDF8E9}$0.025$ & --- & \cellcolor[HTML]{F0F9ED}$0.000$ \\
 & 64 & \cellcolor[HTML]{B1E0AB}$\underline{0.305}$ & --- & --- & --- & --- & --- & --- \\
\midrule
LoRA\textsuperscript{P} & 1 & --- & --- & --- & \cellcolor[HTML]{5DB96B}$\underline{0.572}$ & --- & --- & --- \\
\bottomrule
\end{tabular}
\end{center}

\subsection{Llama-3.1-8B, MATH}

\begin{center}
\textit{Llama-3.1-8B.}\\[2pt]
\begin{tabular}{llccccccc}
\toprule
\textbf{Method} & $r$ & $10^{-6}$ & $3{\cdot}10^{-6}$ & $10^{-5}$ & $3{\cdot}10^{-5}$ & $10^{-4}$ & $3{\cdot}10^{-4}$ & $10^{-3}$ \\
\midrule
LoRA & 1 & --- & --- & --- & \cellcolor[HTML]{94D390}$\underline{0.194}$ & --- & --- & --- \\
 & 8 & --- & --- & --- & \cellcolor[HTML]{C4E8BD}$\underline{0.171}$ & --- & --- & --- \\
 & 32 & --- & --- & --- & \cellcolor[HTML]{65BD6F}$\underline{0.212}$ & --- & --- & --- \\
\midrule
LoRA\textsuperscript{P} & 1 & --- & --- & --- & \cellcolor[HTML]{E9F7E5}$\underline{0.146}$ & --- & --- & --- \\
 & 8 & --- & --- & --- & \cellcolor[HTML]{F0F9ED}$\underline{0.139}$ & --- & --- & --- \\
 & 32 & --- & --- & --- & \cellcolor[HTML]{CBEBC5}$\underline{0.167}$ & --- & --- & --- \\
\midrule
DiReFT\textsuperscript{A} & 1 & --- & --- & --- & --- & \cellcolor[HTML]{95D391}$\underline{0.193}$ & --- & --- \\
 & 8 & --- & --- & --- & --- & \cellcolor[HTML]{349D53}$\underline{0.233}$ & --- & --- \\
 & 32 & --- & --- & --- & --- & \cellcolor[HTML]{2F974E}$\underline{0.236}$ & --- & --- \\
\midrule
DiReFT\textsuperscript{P} & 1 & --- & --- & --- & --- & \cellcolor[HTML]{BAE3B3}$\underline{0.176}$ & --- & --- \\
 & 8 & --- & --- & --- & --- & \cellcolor[HTML]{309950}$\underline{0.235}$ & --- & --- \\
 & 32 & --- & --- & --- & --- & \cellcolor[HTML]{349D53}$\underline{0.233}$ & --- & --- \\
\bottomrule
\end{tabular}
\end{center}

\subsection{Llama-3.1-8B, HumanEval}

\begin{center}
\textit{Llama-3.1-8B.}\\[2pt]
\begin{tabular}{llccccccc}
\toprule
\textbf{Method} & $r$ & $10^{-6}$ & $3{\cdot}10^{-6}$ & $10^{-5}$ & $3{\cdot}10^{-5}$ & $10^{-4}$ & $3{\cdot}10^{-4}$ & $10^{-3}$ \\
\midrule
LoRA & 1 & --- & --- & --- & --- & \cellcolor[HTML]{83CB82}$\underline{0.372}$ & --- & --- \\
 & 8 & --- & --- & --- & --- & \cellcolor[HTML]{349D53}$\underline{0.500}$ & --- & --- \\
 & 32 & --- & --- & --- & \cellcolor[HTML]{3CA559}$\underline{0.482}$ & \cellcolor[HTML]{C4E8BD}$0.256$ & --- & --- \\
\midrule
LoRA\textsuperscript{P} & 1 & --- & --- & --- & --- & \cellcolor[HTML]{7FC97F}$\underline{0.378}$ & --- & --- \\
 & 8 & --- & --- & --- & --- & \cellcolor[HTML]{66BD6F}$\underline{0.415}$ & --- & --- \\
 & 32 & --- & --- & --- & \cellcolor[HTML]{66BD6F}$\underline{0.415}$ & \cellcolor[HTML]{92D28F}$0.348$ & \cellcolor[HTML]{F0F9ED}$0.134$ & --- \\
\midrule
DiReFT\textsuperscript{A} & 1 & --- & --- & --- & --- & --- & \cellcolor[HTML]{6BC072}$\underline{0.408}$ & --- \\
 & 8 & --- & --- & --- & --- & \cellcolor[HTML]{5EB96B}$\underline{0.427}$ & \cellcolor[HTML]{B7E2B1}$0.281$ & --- \\
 & 32 & --- & --- & --- & --- & --- & \cellcolor[HTML]{2F974E}$\underline{0.512}$ & --- \\
\midrule
DiReFT\textsuperscript{P} & 1 & --- & --- & --- & --- & --- & \cellcolor[HTML]{70C274}$\underline{0.402}$ & --- \\
 & 8 & --- & --- & --- & --- & --- & \cellcolor[HTML]{6BC072}$\underline{0.408}$ & --- \\
 & 32 & --- & --- & --- & --- & --- & \cellcolor[HTML]{52B365}$\underline{0.445}$ & --- \\
\bottomrule
\end{tabular}
\end{center}